%% file: main.tex
\documentclass[preprint,12pt,nopreprintline]{elsarticle}

\usepackage{amssymb}
\usepackage{amsthm}
\usepackage{soul}
\usepackage{url}
\usepackage{hyperref}
\usepackage{inputenc}
\usepackage{caption}
\usepackage{float}
\newfloat{model}{htbp}{lom}
\floatname{model}{Model}
\usepackage{mdframed}
\usepackage{graphicx}
\usepackage{booktabs}
\usepackage{subcaption}
\usepackage{hyperref}
\usepackage{setspace}
\usepackage{multirow}
\usepackage{bm,bbm}
\usepackage{enumitem}
\usepackage{amsmath}
\usepackage{amssymb}
\usepackage{amsthm}
\usepackage{mathtools}
\usepackage{amsfonts}
\usepackage{booktabs}       
\usepackage{nicefrac}  
\usepackage{amsthm}
\usepackage[switch]{lineno}
\usepackage[ruled,noresetcount,vlined,linesnumbered]{algorithm2e}
\usepackage{xcolor}

\newcommand{\rev}[1]{{\color{purple}{#1}}}

\DeclareMathOperator*{\argmax}{argmax}
\DeclareMathOperator*{\argmin}{argmin}

\newcommand{\cO}{\mathcal{O}}

\newtheorem{theorem}{Theorem}
\newtheorem{definition}{Definition}
\newtheorem{corollary}{Corollary}
\newtheorem{lemma}{Lemma}

\journal{Artificial Intelligence}

\allowdisplaybreaks
\begin{document}

\begin{frontmatter}

\title{Analyzing and Enhancing the Backward-Pass Convergence of Unrolled Optimization}

\author{James Kotary}
\ead{jk4pn@virginia.edu}
\author{Jacob K Christopher}
\ead{csk4sr@virginia.edu}
\author{My H Dinh}
\ead{fqw2tz@virginia.edu}
\author{Ferdinando Fioretto}
\ead{fioretto@virginia.edu}

\address{University of Virginia, Charlottesville, Virginia, USA}

\begin{abstract}
The integration of constrained optimization models as components in deep networks has led to promising advances on many specialized learning tasks. 
A central challenge in this setting is backpropagation through the solution of an optimization problem, which often lacks a closed form. One typical strategy is algorithm unrolling, which relies on automatic differentiation through the entire chain of operations executed by an iterative optimization solver.   This paper provides theoretical insights into the backward pass of unrolled optimization, showing that it is asymptotically equivalent to the solution of a linear system by a particular iterative method. Several practical pitfalls of unrolling are demonstrated in light of these insights, and a system called Folded Optimization is proposed to construct more efficient backpropagation rules from unrolled solver implementations. Experiments over various end-to-end optimization and learning tasks demonstrate the advantages of this system both computationally, and in terms of flexibility over various optimization problem forms.  

\noindent\textbf{Code available at:} 
\href{https://fold-opt.github.io}
{https://fold-opt.github.io}

\end{abstract}


\end{frontmatter}

\section{Introduction}
\label{sec:introduction}

The integration of optimization problems as components in neural networks has shown to be an effective framework for enforcing structured representations in deep learning. A parametric optimization problem defines a mapping from its unspecified parameters to the resulting optimal solutions, which is treated as a layer of a neural network. By allowing neural networks to learn over the space of parametrized optimal solutions, optimization as a layer can offer enhanced accuracy and efficiency on specialized learning tasks, by imparting predefined task-specific structure to their representations \citep{kotary2021end}.



For example, optimization layers can generalize the functionality of earlier structured prediction mechanisms for tasks such as multilabel classification \citep{martins2016softmax} and learning to rank \citep{adams2011ranking, kotary2022end} using simple optimization models. The integration of operational decision problems as components in neural networks has shown promise in enhancing the effectiveness of data-driven decision models \citep{wilder2019melding}. Some work has even shown that optimization components with learnable constraints can be used as general-purpose layers, capable of greater expressiveness than conventional linear layers \citep{amos2019optnet}.

While these mechanisms can be used in much the same way as linear layers and activation functions, the resulting  end-to-end models require training by stochastic gradient descent, which in turn requires differentiation of the optimization mappings for backpropagation of gradients. This poses unique challenges, partly due to their lack of a closed form, and modern approaches typically follow one of two strategies: In \emph{unrolling}, an optimization algorithm is executed entirely on the computational graph, and backpropagated by automatic differentiation from optimal solutions to the underlying problem parameters. The approach is adaptable to many problem classes, but has been shown to suffer from time and space inefficiency, as well as vanishing gradients \citep{monga2021algorithm}. \emph{Analytical differentiation} is a second strategy that circumvents those issues by forming implicit models for the derivatives of an optimization mapping and solving them exactly. However, this requires construction of the problem-specific implicit derivative models, and the most popular current framework for its automation puts rigid requirements on the form of the optimization problems, relying on transformations to canonical convex cone programs before applying a standardized procedure for their solution and differentiation  \citep{agrawal2019differentiable}.  This precludes the use of specialized solvers that are best-suited to handle various optimization problems, and inherently restricts itself only to convex problems. 

\paragraph{\bf Contributions}
This paper presents a novel analysis of unrolled optimization, which results in efficiently-solvable models for the backpropagation of unrolled optimization. Theoretically, the result is significant because it reveals an equivalence between unrolling and analytical differentiation, and allows for convergence of the backward pass to be analyzed. Practically, it allows for the forward and backward passes of unrolled optimization to be disentangled and solved separately, using generic implementations of specialized algorithms. More specifically, this paper makes the following novel contributions\footnote{This paper is an extended version of IJCAI-23 paper \citep{kotary2023folded}. It expands on the conference version substantially, with new material detailing extensions to the folded optimization library, and a new collection of experiments which study the backward-pass convergence of unrolled optimization and empirically illustrate the improvements therein which are made possible by the paper's proposed framework.}:    

\begin{enumerate}[leftmargin=*, parsep=0pt, itemsep=0pt, topsep=2pt]
\item A theoretical analysis of unrolling, which shows that its backward pass is asymptotically equivalent to the solution of a linear system of equations by a particular iterative method, and allows for its convergence properties to be quantified. 
\item An empirical evaluation of the backward-pass convergence behavior of unrolled optimization, which corroborates several aspects of the aforementioned theory and indicates several potential pitfalls that arise as a result of its naive implementation.
\item Building on this analysis, the paper proposes a system for generating  analytically differentiable optimization solvers from unrolled implementations called \emph{folded optimization}, accompanied by a Python library called \texttt{fold-opt} to facilitate automation, available at:\\ \href{https://fold-opt.github.io}{https://fold-opt.github.io}.
\item The efficiency and modeling advantages of folded optimization are demonstrated on a diverse set of end-to-end optimization and learning tasks, which include end-to-end learning with difficult \emph{nonconvex} optimization.
\end{enumerate}

\section{Setting and Goals}
\label{sec:SettingAndGoals}

In this paper, the goal is to differentiate mappings that are defined as the solution to an optimization problem.  Consider the parameterized problem  (\ref{eq:opt_generic}) which defines a function from a vector of parameters $\mathbf{c} \in \mathbb{R}^p$ to its associated optimal solution $\mathbf{x}^{\star}( \mathbf{c} ) \in \mathbb{R}^n$:     
\begin{subequations} 
    \label{eq:opt_generic}
    \begin{align}
        \mathbf{x}^{\star}(\mathbf{c}) =  \argmin_{\mathbf{x}} &\; f(\mathbf{x},\mathbf{c})  \label{eq:opt_generic_objective}\\        
        \text{subject to: }&\;
        g(\mathbf{x},\mathbf{c}) \leq \mathbf{0},  \label{eq:opt_generic_inequality}\\
        &\; h(\mathbf{x},\mathbf{c}) = \mathbf{0},\label{eq:opt_generic_equality}
    \end{align}
\end{subequations}
in which $f$ is the objective function, and $g$ and $h$ are vector-valued functions capturing the inequality and equality constraints of the problem, respectively. The parameters $\mathbf{c}$ can be thought of as a prediction from previous layers of a neural network, or as learnable parameters analogous to the weights of a linear layer, or as some combination of both.   It is assumed throughout that for any $\mathbf{c}$, the associated optimal solution $\mathbf{x}^{\star}(\mathbf{c})$ can be found by conventional methods, within some tolerance in solver error. 
This coincides with the ``forward pass'' of the mapping in a neural network. \emph{The primary challenge is to compute its backward pass}, which amounts to finding the Jacobian matrix of partial derivatives $\frac{\partial \mathbf{x}^{\star}(\mathbf{c})}{\partial \mathbf{c}}$ .

\paragraph{\bf Backpropagation}  
Given a downstream task loss $\mathcal{L}$, backpropagation through   $\mathbf{x}^{\star}(\mathbf{c})$ amounts to computing $\frac{\partial \mathcal{L}}{\partial \mathbf{c}}$ given $\frac{\partial \mathcal{L}}{\partial \mathbf{x}^{\star}}$.  In deep learning, backpropagation through a layer is typically accomplished by automatic differentiation (AD), which propagates gradients through the low-level operations of an overall composite function by repeatedly applying the multivariate chain rule. This can be performed automatically given a forward pass implementation in an AD library such as PyTorch. However, it requires a record of all the operations performed during the forward pass and their dependencies, known as the \emph{computational graph}.
 
\paragraph{\bf Jacobian-gradient product (JgP)}
The \emph{Jacobian} matrix of the vector-valued function $\mathbf{x}^{\star}(\mathbf{c}): \mathbb{R}^p \rightarrow \mathbb{R}^n$ is a matrix $\frac{\partial \mathbf{x}^{\star}}{\partial \mathbf{c}}$ in $\mathbb{R}^{n \times p}$, whose elements at $(i,j)$  are the partial derivatives $\frac{\partial x^{\star}_i(\mathbf{c})}{\partial c_j}$. When the Jacobian is known, backpropagation through $\mathbf{x}^{\star}(\mathbf{c})$  can be performed  by computing the product
\begin{equation}
\label{eq:jvprod}
  \frac{\partial \mathcal{L}}{\partial \mathbf{c}} = \frac{\partial \mathcal{L}}{\partial \mathbf{x}^{\star}} \cdot \frac{\partial \mathbf{x}^{\star}(\mathbf{c})}{\partial \mathbf{c}}.
 \end{equation}

\begin{figure}[t]
    \centering
    \includegraphics[width=1.0\linewidth]{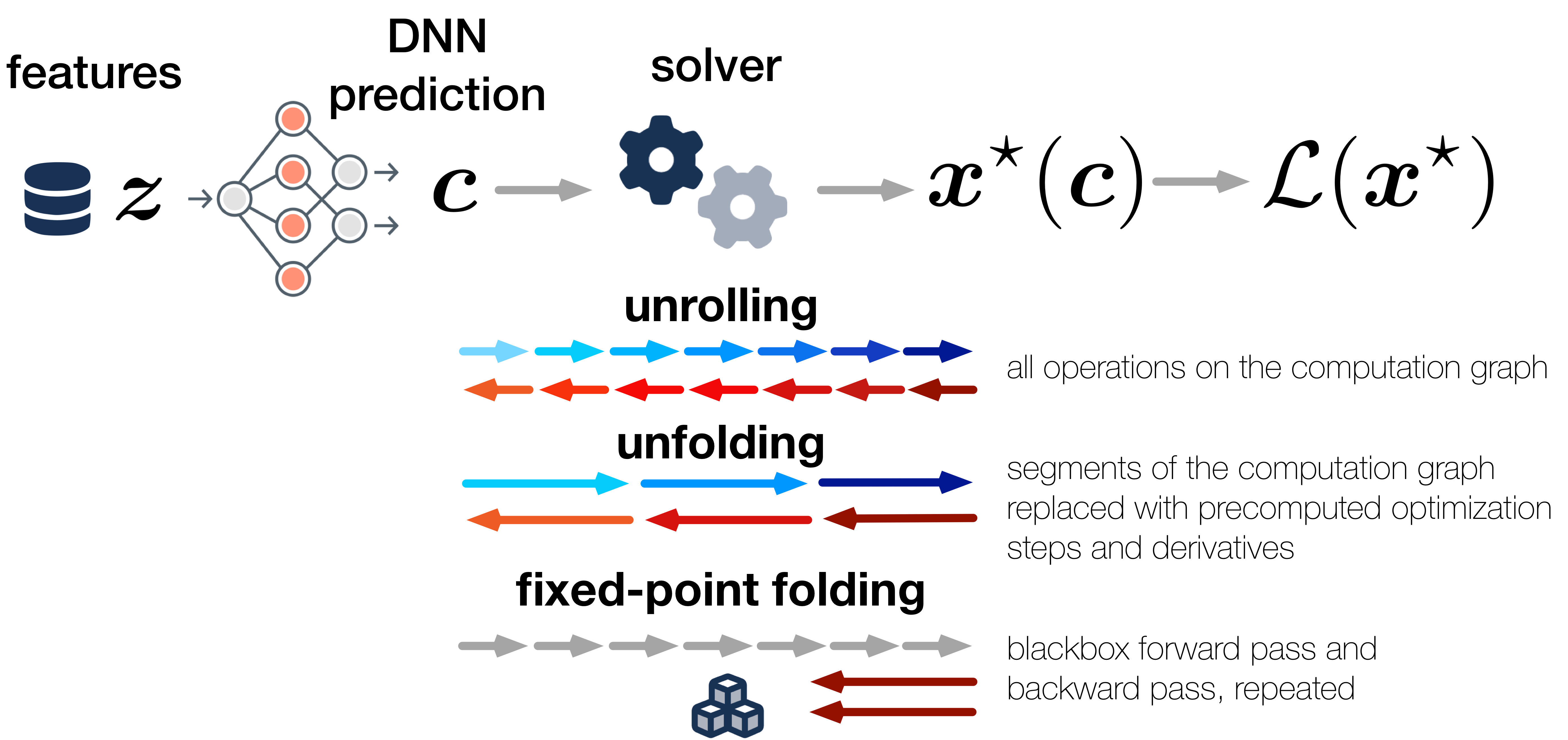}
    \caption{Compared to unrolling, unfolding requires fewer operations on the computational graph by replacing inner loops with Jacobian-gradient products.   Fixed-point folding models the unfolding analytically, allowing for generic implementations.}
    \label{fig:unfolding_scheme}
\end{figure}

\section{Folded Optimization: Overview}
The problem \eqref{eq:opt_generic} is most often solved by iterative methods, which refine an initial \emph{starting point} $\mathbf{x}_0$ by repeated application of a subroutine, which we view as a function. For optimization variables $\mathbf{x} \in \mathbb{R}^n$, the \emph{update function} is  a vector-valued function $\mathcal{U}: \mathbb{R}^n \rightarrow \mathbb{R}^n$:
\begin{equation}
    \label{eq:opt_iteration}\tag{U}
        \mathbf{x}_{k+1}(\mathbf{c}) = \mathcal{U}(\mathbf{x}_k (\mathbf{c}) ,\;  \mathbf{c} ).
\end{equation}
\noindent 
The iterations (\ref{eq:opt_iteration}) \emph{converge} if $\mathbf{x}_{k}(\mathbf{c}) \rightarrow \mathbf{x}^{\star}(\mathbf{c})$ as $k \rightarrow \infty$; in the present paper, this is referred to as \emph{forward-pass convergence}.
When the iterations \eqref{eq:opt_iteration} are \emph{unrolled}, they are computed and backpropagated on the computational graph, and the overall function $\mathbf{x}^{\star}(\mathbf{c})$ is thereby backpropagated by AD without explicitly representing its Jacobian matrix $\frac{\partial \mathbf{x}^{\star}(\mathbf{c})}{\partial \mathbf{c}}$. The backpropagation of the unrolled solution process is also an iterative procedure, and we aim to analyze its convergence. To this end, we define convergence of the backward pass in unrolling  as follows:

\begin{definition}
\label{def:bwd_conv}
Suppose that an unrolled iteration (\ref{eq:opt_iteration}) produces a convergent sequence of solution iterates $\lim_{k \to \infty} \mathbf{x}_k  = \mathbf{x}^{\star}$ in its forward pass. Then convergence of the backward pass is defined as the condition 
\begin{equation}
\label{eq:limit-x-diff}
\lim_{k \to \infty} \frac{\partial \mathbf{x}_k}{\partial \mathbf{c}}(\mathbf{c})  = \frac{\partial \mathbf{x}^{\star}}{\partial \mathbf{c}}(\mathbf{c}),
\end{equation}
assuming that all requisite derivatives exist.
\end{definition}

Unrolling \eqref{eq:opt_iteration} over many iterations often faces time and space inefficiency issues due to the need for storage and traversal of the optimization procedure's entire computational graph  \cite{monga2021algorithm}. The following sections analyze the backward pass of unrolled optimization to yield equivalent analytical models for the Jacobian $\frac{\partial \mathbf{x}^{\star}}{\partial \mathbf{c}}(\mathbf{c})$, and show how to efficiently solve those models by leveraging the backward pass of \eqref{eq:opt_iteration}. To do so, we recognize two key challenges in modeling the backward pass of unrolling iterations \eqref{eq:opt_iteration}. First, it often happens that evaluation of $\mathcal{U}$ in \eqref{eq:opt_iteration} requires the solution of another optimization subproblem, such as a projection or proximal operator, which must also be unrolled. Section \ref{sec:unfolding} introduces \emph{\textbf{unfolding}} as a variant of unrolling, in which the unrolling of such inner loops is circumvented by analytical differentiation of the subproblem, allowing its analysis to be confined to a single unrolled loop. 

Second, the backward pass of an unrolled solver is determined by its forward pass, whose trajectory depends on its (potentially arbitrary) starting point and the convergence properties of the chosen algorithm. Section \ref{sec:Unfolding_at_a_fixed_point} shows that the backward pass converges correctly even when the forward-pass iterations are initialized at the precomputed optimal solution. This allows for separation of the forward and backward passes, which are typically entangled across unrolled iterations, greatly simplifying the backward pass model and allowing for generic implementations of both passes.

Section \ref{sec:Unfolding_at_a_fixed_point} uses these concepts to show that the backward pass of unfolding \eqref{eq:opt_iteration} follows the solution, by fixed-point iteration, of the linear system for $\frac{\partial \mathbf{x}^{\star}(\mathbf{c})}{\partial \mathbf{c}}$ which arises by differentiating the fixed-point conditions of \eqref{eq:opt_iteration}. This allows for the conditions and rate of its convergence to be analyzed. Section \ref{sec:Experiments} then outlines \emph{\textbf{folded optimization}}, a system for generating Jacobian-gradient products through optimization based on efficient solution of the models proposed in Section \ref{sec:Unfolding_at_a_fixed_point}. The main differences between unrolling, unfolding, and folded optimization are illustrated in Figure \ref{fig:unfolding_scheme}.

\section{Related Work}
\label{sec:RelatedWork}
We end-to-end optimization and learning approaches into those based on \emph{unrolling}, and \emph{analytical} differentiation. Since this paper focuses on converting unrolled implementations into analytical ones, each category is reviewed first below.

\paragraph{\bf Unrolling optimization algorithms}
Automatic Differentiation (AD) is  the primary method of backpropagating gradients in deep learning models for training with stochastic gradient descent.  Modern machine learning frameworks such as PyTorch have natively implemented differentiation rules for a variety of functions that are commonly used in deep models, as well as interfaces to define custom differentiation rules for new functions \citep{paszke2017automatic}. As a mainstay of deep learning, AD is also a natural tool for backpropagating through constrained optimization mappings. \emph{Unrolling} refers to the execution of an optimization algorithm, entirely on the computational graph, for backpropagation by AD from the resulting optimal solution to its input parameters. Such approaches are general and apply to a broad range of optimization models. They  can be performed simply by implementing a solution algorithm within an AD framework, without the need for analytical modeling of an optimization mapping's derivatives \citep{domke2012generic}. However, unrolling over many iterations has been shown to encounter issues of time and memory inefficiency due to the size of its computational graph \citep{amos2019optnet}. Further issues encountered in unrolling, such as vanishing and exploding gradients, are reminiscent of recurrent neural networks \citep{monga2021algorithm}. On the other hand, unrolling may offer some unique practical advantages, like the ability to learn optimization parameters such as stepsizes to accelerate the solution of each optimization during training \citep{shlezinger2022model}.

\paragraph{\bf Analytical differentiation of optimization models}

Differentiation through constrained argmin problems  in the context of machine learning was discussed as early as \cite{gould2016differentiating}, who proposed first to implicitly differentiate the argmin of a smooth, unconstrained convex function by its first-order optimality conditions, defined when the gradient of the objective function equals zero. This technique is then extended to find approximate derivatives for constrained problems, by applying it to their unconstrained log-barrier approximations. Subsequent approaches applied implicit differentiation to the KKT optimality conditions of constrained problems directly \citep{amos2019optnet,amos2019limited}, but only on special problem classes such as  Quadratic Programs. \cite{konishi2021end} extend the method of \cite{amos2019optnet}, by modeling second-order derivatives of the optimization for training with gradient boosting methods. \cite{donti2017task} uses the differentiable quadratic programming solver of \citep{amos2019optnet} to approximately differentiate general convex programs through quadratic surrogate problems.   Other problem-specific approaches to analytical differentiation models include ones for sorting and ranking \citep{blondel2020fast}, linear programming \citep{mandi2020interior}, and convex cone programming \citep{agrawal2019differentiating}.

The first general-purpose differentiable optimization solver was proposed in \cite{agrawal2019differentiable}, which leverages the fact that any convex program can be converted to a convex cone program \citep{nemirovski2007advances}. The equivalent cone program is subsequently solved and differentiated following \cite{agrawal2019differentiating}, which implicitly differentiates a zero-residual condition representing optimality \citep{busseti2019solution}. A differentiable solver library \texttt{cvxpy} is based on this approach, which converts convex programs to convex cone programs by way of their graph implementations as described in \cite{grant2008graph}. The main advantage of the system is that it applies to any convex program and has a simple symbolic interface. A major disadvantage is its restriction to solving problems only in a standard convex cone form with an ADMM-based conic programming solver, which performs poorly on some problem classes, as seen in Section \ref{sec:Experiments}. 

A related line of work concerns end-to-end learning with \emph{discrete} optimization problems, which includes linear programs, mixed-integer programs and constraint programs. These problem classes often define discontinuous mappings with respect to their input parameters, making their true gradients unhelpful as descent directions in  optimization. Accurate end-to-end training can be achieved by \emph{smoothing} the optimization mappings, to produce approximations which yield more useful gradients. A common approach is to augment the objective function with smooth regularizing terms such as euclidean norm or entropy functions \citep{wilder2019melding,ferber2020mipaal,mandi2020interior}. Others show that similar effects can be produced by applying random noise to the objective \citep{berthet2020learning,paulus2020gradient}, or through finite difference approximations \citep{poganvcic2019differentiation,sekhar2022gradient}. This enables end-to-end learning with discrete structures such as constrained ranking policies \citep{kotary2022end}, shortest paths in graphs \citep{elmachtoub2020smart}, and various decision models \citep{wilder2019melding}.

\section{From Unrolling to Unfolding}
\label{sec:unfolding}
For many optimization algorithms of the form (\ref{eq:opt_iteration}), the update function $\mathcal{U}$ is composed of closed-form functions that are relatively simple to evaluate and differentiate. In general though, $\mathcal{U}$ may itself employ an optimization subproblem that is nontrivial to differentiate. That is, 
\begin{equation}
    \label{eq:inner_optimization}\tag{O}
        \mathcal{U} ( \mathbf{x}_k   ) \coloneqq {\mathcal{T}} \left( \;  \mathcal{O}(\mathcal{S}(\mathbf{x}_k)), \;\;  \mathbf{x}_k  \;\right),
\end{equation}
\noindent 
wherein the differentiation of $\mathcal{U}$ is complicated by an \emph{inner optimization} sub-routine $\mathcal{O}: \mathbb{R}^n \rightarrow \mathbb{R}^n$. 
Here, $\mathcal{S}$ and $\mathcal{T}$ represent any steps preceding or following the inner optimization (such as gradient steps), viewed as closed-form functions. In such cases, unrolling (\ref{eq:opt_iteration}) would also require unrolling  $\mathcal{O}$. If the Jacobians of $\mathcal{O}$ can be found, then backpropagation through $\mathcal{U}$ can be completed, free of unrolling, by applying a chain rule through Equation (\ref{eq:inner_optimization}), which in this framework is handled naturally by automatic differentiation of $\mathcal{T}$ and $\mathcal{S}$.

Then, only the outermost iterations (\ref{eq:opt_iteration}) need be unrolled on the computational graph for backpropagation. This partial unrolling, which allows for backpropagating large segments of computation at a time by leveraging analytically differentiated subroutines, is henceforth referred to as \emph{unfolding}. It is made possible when the update step $\mathcal{U}$ is easier to differentiate than the overall optimization mapping $\mathbf{x}^{\star}(\mathbf{c})$.
\begin{definition}[Unfolding]
\label{def:unfolding}
An \emph{unfolded} optimization of the form (\ref{eq:opt_iteration}) is one in which the backpropagation of $\mathcal{U}$ at each step does not require unrolling an iterative algorithm.
\end{definition}
Unfolding is distinguished from more general unrolling by the presence of only a single unrolled loop. This definition sets the stage for Section \ref{sec:FoldedOptimization}, which shows how the backpropagation of an unrolled loop can be modeled with a Jacobian-gradient product. Thus, unfolded optimization is a precursor to the complete replacement of backpropagation through loops in unrolled solver implementations by JgP. 

When $\mathcal{O}$ has a closed form and does not require an iterative solution, the definitions unrolling and unfolding coincide. When $\mathcal{O}$ is nontrivial to solve but has known Jacobians, they can be used to produce an unfolding of (\ref{eq:opt_iteration}). Such is the case when $\mathcal{O}$ is a Quadratic Program (QP); a JgP-based differentiable QP solver called \texttt{qpth} is provided by  \cite{amos2019optnet}. Alternatively, the replacement of unrolled loops by JgP's proposed in Section \ref{sec:FoldedOptimization} can be applied recursively $\mathcal{O}$.

\begin{figure}[t]
    \centering
    \includegraphics[width=0.70\linewidth]{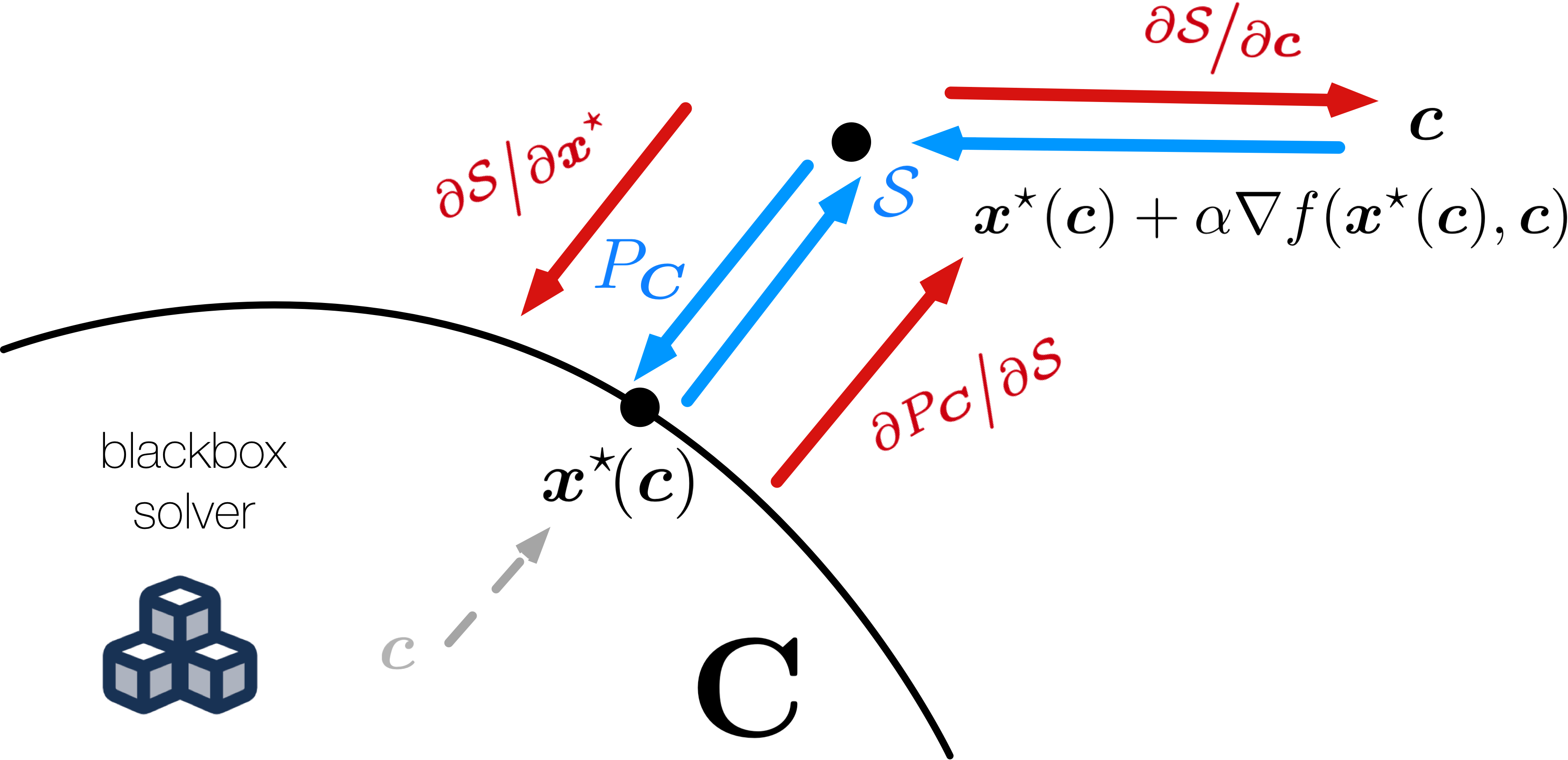}
    \caption{Unfolding Projected Gradient Descent at $\mathbf{x}^{\star}$ consists of alternating gradient step $\mathcal{S}$ with projection $\mathcal{P}_{\mathbf{C}}$. Section \ref{sec:Unfolding_at_a_fixed_point} shows that the resulting chain of JgP operations in backpropagation is equivalent to solving the differential fixed-point conditions \eqref{eq:Lemma_fixedpt} by linear fixed-point iteration. Each function's forward and backward pass are illustrated in blue and red, respectively.  }
    \label{fig:unfold_pgd}
\end{figure}

These concepts are illustrated in the following examples, highlighting the roles of $\mathcal{U}$, $\mathcal{O}$  and $\mathcal{S}$. Each will be used to create folded optimization mappings for a variety of learning tasks in Section \ref{sec:Experiments}.

\paragraph{Projected gradient descent}
Given a problem 
\begin{equation}
\label{eq:problem_pgd}
        \min_{\mathbf{x} \in \mathbf{C}} \; f(\mathbf{x})   
\end{equation}
where $f$ is differentiable and $\mathbf{C}$ is the feasible set, Projected Gradient Descent (PGD) follows the update function
\begin{equation}
\label{eq:update-pgd}
        \mathbf{x}_{k+1} =   \mathcal{P}_{\mathbf{C}}( \mathbf{x}_k - \alpha_k \nabla f (\mathbf{x}_k) ),
\end{equation}
where $\mathcal{O} = \mathcal{P}_{\mathbf{C}}$ is the Euclidean projection onto $\mathbf{C}$, and $\mathcal{S}(\mathbf{x}) = \mathbf{x} - \alpha \nabla f (\mathbf{x}) $ is a gradient descent step. Many simple $\mathbf{C}$ have closed-form projections to facilitate unfolding of (\ref{eq:update-pgd}) (see \citep{beck2017first}). Further, when $\mathbf{C}$ is linear, $\mathcal{P}_\mathbf{C}$ is a quadratic programming (QP) problem for which a differentiable solver \texttt{qpth} is available from \cite{amos2019optnet}.

Figure \ref{fig:unfold_pgd} shows one iteration of unfolding projected gradient descent, with  the forward and backward pass of each recorded operation on the computational graph illustrated in blue and red, respectively. 

\paragraph{Proximal gradient descent}
More generally, to solve 
\begin{equation}
\label{eq:problem_prox}
        \min_{\mathbf{x}} \; f(\mathbf{x}) + g(\mathbf{x}) 
\end{equation}
where $f$ is differentiable and $g$ is a closed convex function, proximal gradient descent follows the update function
\begin{equation}
\label{eq:update-prox}
        \mathbf{x}_{k+1} =   \operatorname{Prox}_{\alpha_k g}
        \left(\mathbf{x}_k - \alpha_k \nabla f (\mathbf{x}_k) \right).
\end{equation}
Here $\mathcal{O}$ is the proximal operator, defined as
\begin{equation}
\label{eq:def-prox}
\operatorname{Prox}_{g}(\mathbf{x}) = \argmin_{y} 
\left\{g(\mathbf{y}) + \frac{1}{2} 
\| \mathbf{y} - \mathbf{x}  \|^2 \right\},
\end{equation}
and its difficulty depends on $g$. Many simple proximal operators can be represented in closed form and have simple derivatives. 
For example, 
when $g(\mathbf{x}) = \lambda \|\mathbf{x}\|_1$, then $\operatorname{Prox}_{g} = \mathcal{T}_{\lambda}(\mathbf{x})$ is the soft thresholding operator, whose closed-form formula and derivative are given in \ref{appendix:models}.

\paragraph{Sequential quadratic programming}
Sequential Quadratic Programming (SQP)  solves the general optimization problem (\ref{eq:opt_generic}) by approximating it at each step by a QP problem, whose objective is a second-order approximation of the problem's Lagrangian function, subject to a linearization of its constraints. SQP is well-suited for  unfolded optimization, as it can solve a broad class of convex and nonconvex problems and can readily be unfolded by implementing its QP step (shown in  \ref{appendix:models}) 
with the  \texttt{qpth} differentiable QP solver.

\paragraph{Quadratic programming by ADMM}
The  QP solver of \cite{boyd2011distributed}, based on the alternating direction of multipliers, is specified in  \ref{appendix:models}. 
Its inner optimization step $\mathcal{O}$ is a simpler equality-constrained QP; its solution is equivalent to solving a linear system of equations, which has a simple derivative rule in PyTorch.

\smallskip
Given an unfolded QP solver by ADMM, its unrolled loop can be replaced with backpropagation by JgP as shown in Section \ref{sec:FoldedOptimization}. The resulting  differentiable QP solver can then take the place of \texttt{qpth} in the examples above. Subsequently, \emph{this technique can be applied recursively} to the resulting unfolded PGD and SQP solvers. This exemplifies the intermediate role of unfolding in converting unrolled, nested solvers to fully JgP-based implementations, detailed in Section~\ref{sec:Experiments}.

From the viewpoint of unfolding, the analysis of backpropagation in unrolled solvers can be simplified by accounting for only a single unrolled loop at a time. The next section identifies a further simplification: \emph{that the backpropagation of an unfolded solver can be completely characterized by its action at a fixed point of the solution's algorithm.} 

\section{Unfolding at a Fixed Point}
\label{sec:Unfolding_at_a_fixed_point}

Optimization methods of the form  (\ref{eq:opt_iteration}) require a  starting point $\mathbf{x}_0$, which is often chosen arbitrarily, since (forward-pass) convergence $\mathbf{x}_k \to \mathbf{x}^{\star}$  is typically ensured regardless of  $\mathbf{x}_0$. In unfolded optimization, it is natural to also ask how the choice of $\mathbf{x}_0$ affects the backward-pass convergence. Here, the special case when $\mathbf{x}_0 = \mathbf{x}^{\star}$ is of particular interest. In this case, the forward pass of an unrolled optimization is equivalent to an identity function at each iteration, since $\mathbf{x}^{\star}$ is a fixed point of $\cal U$.  Therefore if the backward pass  converges in this case (as per Definition \ref{def:bwd_conv}), it can be considered as a standalone procedure, independent of the optimization method's forward pass. This separation of the forward and backward passes is key to an enhanced system of backpropagation, called \emph{folded optimization}, introduced in Section \ref{sec:FoldedOptimization}.

In this section, we first demonstrate empirically that the backward pass of unfolded optimization does in fact converge, resulting in correct gradients when the starting point $\mathbf{x}_0 \to \mathbf{x}^{\star}$ is chosen. Then a theoretical analysis is presented, which shows that the principle holds in general, and allows the backward-pass convergence to be analyzed. A more thorough empirical study is then discussed, which corroborates the main theoretical results in practice, while illustrating the disadvantages and potential pitfalls inherent to naive unfolding implementations. Section \ref{sec:FoldedOptimization} then shows how the results of this Section form the basis of folded optimization, a more efficient and reliable system for backpropagation which builds on the idea of unfolding at a precomuputed fixed point.

\begin{figure}
    \centering
    \includegraphics[width=0.99\linewidth]{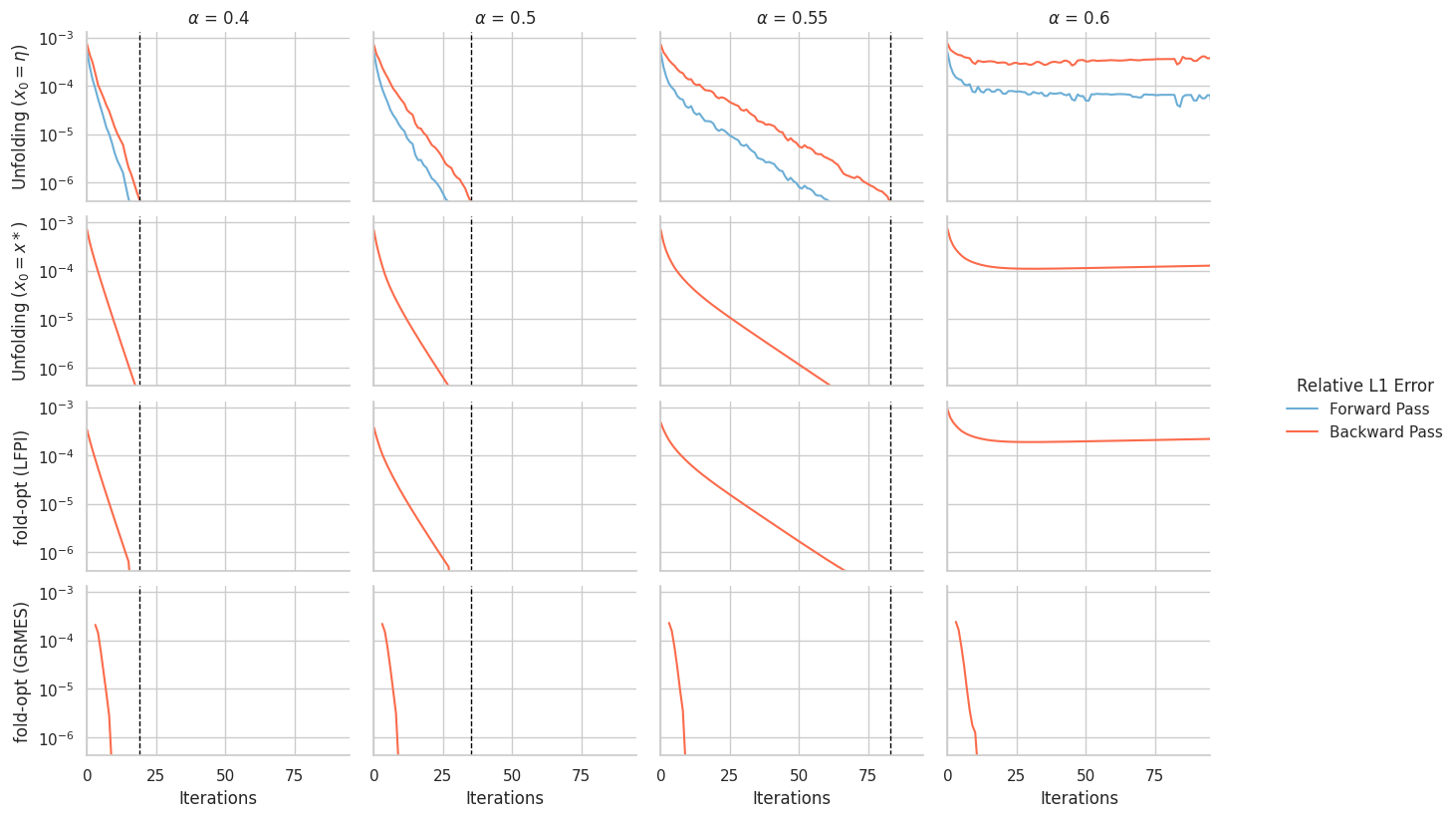} 
    \\
    \hfill
    \caption{Forward and backward pass error per number of iteratons, across different stepsizes on CIFAR100 Multilabel Classification. Error is measured on average over $100$ samples. Each row represents a distinct differentiable solver implementation; the first two represent unfolded PGD and the latter two represent folded optimization counterparts. Columns correspond to PGD stepsize. }
    \label{fig:fwd_bwd_err}
\end{figure}

\subsection{An illustrative example} 
\label{sec:illustrative_example}
The empirical results of this section are based on a representative example problem, in which the forward and backward pass errors are measured at each iteration of an unfolded solver.  The optimization problem (\ref{eq:topk-lp}) maps feature embeddings $\bf{c}$ to smoothed top-$k$ class indicators $\bf{x}^{\star}$, and is used to learn multilabel classification later in Section \ref{sec:Experiments}. Unfolded projected gradient descent is used to differentiably compute the mapping $\mathbf{c} \rightarrow \bf{x}^{\star}(\mathbf{c})$, in which the projection onto linear constraints is computed and backpropagated using the differentiable \texttt{qpth} QP solver.
A loss function $\mathcal{L}$ targets ground-truth top-$k$ indicators, and the result of the backward pass is an estimate of the gradient $\frac{\partial \mathcal{L}}{\partial \mathbf{c}}$.  
We evaluate the forward and backward-pass convergence of unfolded projected gradient descent, by measuring the relative $L_1$ errors of the forward and backward passes, relative to the true optimal solutions and corresponding loss gradients. 

Two types of starting points are considered: the precomputed optimal solution $\mathbf{x}^a_0 = \mathbf{x}^{\star}$, and a uniform random vector $\mathbf{x}^b_0 = \mathbf{\eta} \sim \mathbf{U}(0,1)$. In the latter case, the error is reported on average over $20$ random starting points. The former case is illustrated in Figure \ref{fig:unfold_pgd}, in which $\mathbf{x}_k$  remains stationary at each step of projected gradient descent in its forward pass.  In addition, four different \emph{fixed} gradient stepsizes $\alpha \in \{ 0.4, 0.5, 0.55, 0.6 \}$ are considered. Figure \ref{fig:fwd_bwd_err} plots the relative $L_1$ errors of the forward pass (in blue) and backward pass (in red) for $0 \leq n \leq 100$ iterations of unfolded PGD under the two starting points and various stepsizes. The first two rows correspond to unfolding PGD in the cases $\mathbf{x}_0 = \mathbf{\eta}$ and $\mathbf{x}_0 = \mathbf{x}^{\star}$; the last two rows correspond to two folded optimization variants and are discussed later. 

The absence of blue curves in the case $\mathbf{x}_0 = \mathbf{x}^{\star}$ indicates that when starting the unfolding from the precomputed optimal solution, the forward pass error remains near zero. This behavior is expected, since $\mathbf{x}^{\star}(\mathbf{c}) \!=\! \mathcal{U}( \mathbf{x}^{\star}(\mathbf{c}), \mathbf{c} )$ is a fixed point of (\ref{eq:opt_iteration}). On the other hand, the figure's red curves show that for each chosen stepsize $\alpha$, the backward pass of unfolding converges whenever the forward pass converges, whether it is initialized at a fixed point or a random point. Further, the rate of backward-pass convergence is highly dependent on the chosen stepsize $\alpha$, even when $\bf{x}_0 = \bf{x}^{\star}$. The case $\alpha = 0.6$ indicates a critical point, beyond which the unfolding fails to converge. 

We observe from Figure \ref{fig:fwd_bwd_err} two major trends in the convergence patterns of the backward pass: \textbf{(1)} The vertical dotted lines of Figure \ref{fig:fwd_bwd_err} aid comparison against time to convergence in the backward pass which results from the choice of $\bf{x}_0 = \mathbf{\eta}$. Backward-pass convergence is always faster in the case $\bf{x}_0 = \bf{x}^{\star}$, and this effect becomes more pronounced for the $\alpha$ which result in slower convergence. \textbf{(2)} The convergence of the backward pass always lags behind that of the forward pass in the case where $\bf{x}_0 = \bf{\eta}$, and this effect also becomes more pronounced as the choice of $\alpha$ leads to slower convergence. 

Unfolding in the case when $\bf{x}_0 = \bf{x}^{\star}$ is  referred to as \emph{fixed-point unfolding}. While its backward pass tends to converge faster than that of general unfolding, its requirement of both the precomputed solution $\bf{x}^{\star}$ and the unfolded iterations \ref{eq:opt_iteration} make it impractical in terms of efficiency. However, as shown next, fixed-point unfolding forms an important conceptual starting point for understanding the backward-pass convergence of unfolded optimization in general, and for the more efficient \emph{folded} optimization system introduced in Section \ref{sec:FoldedOptimization}.

\subsection{Backward Convergence of Fixed-Point Unfolding}
\label{sec:convergence}
Next, it will be shown that backpropagation of unfolded optimization at a fixed point is equivalent to solving a linear system of equations for the backpropagated gradients, using a particular iterative method for linear systems. In order to prove this, the following two Lemmas respectively identify the iterative solution method, and the linear system it solves. The following textbook result can be found, e.g., in \citep{quarteroni2010numerical}.

\begin{lemma}
\label{lemma:linear-iteration}
Let $\mathbf{B} \in \mathbb{R}^{n \times n}$ and $\mathbf{b} \in \mathbb{R}^{n}$.  For any $\mathbf{z}_0 \in \mathbb{R}^n$, the iteration 
\begin{equation}
\label{eq:linear-iteration} \tag{LFPI}
        \mathbf{z}_{k+1}  =
         \mathbf{B} \mathbf{z}_{k} + 
        \mathbf{b} 
\end{equation}
converges to the solution $\mathbf{z}$ of the linear system
\( \mathbf{z} = \mathbf{B} \mathbf{z} +  \mathbf{b} \)
whenever $\mathbf{B}$ is nonsingular and has spectral radius
$\rho(\mathbf{B}) < 1$. 
Furthermore, the asymptotic convergence rate for $\mathbf{z}_k \to \mathbf{z}$ is 
\begin{equation}
\label{eq:convergence-rate}
    -\log\; \rho(\mathbf{B}) .
\end{equation}
\end{lemma}

\noindent 
Linear fixed-point iteration (LFPI) is a foundational iterative linear system solver, and can be applied to any linear system $\mathbf{A} \mathbf{x} \!=\! \mathbf{b}$ by rearranging 
$\mathbf{z} \!=\! \mathbf{B} \mathbf{z} \!+\! \mathbf{b}$ and identifying $\mathbf{A} \!=\! \mathbf{I} \!-\! \mathbf{B}$. 

Next, we exhibit the linear system which is solved for the desired gradients $\frac{\partial \mathbf{x}^{\star}}{\partial \mathbf{c}} (\mathbf{c})$ by unfolding at a fixed point. Consider the fixed-point  conditions of the iteration (\ref{eq:opt_iteration}):
\begin{equation}
    \label{eq:fixedpt_cond} \tag{FP}
    \mathbf{x}^{\star}(\mathbf{c}) = \mathcal{U}(\mathbf{x}^{\star}(\mathbf{c}), \; \mathbf{c} )
\end{equation}

Differentiating (\ref{eq:fixedpt_cond}) with respect to $\mathbf{c}$, we define the Jacobians $\mathbf{\Phi}$ and $\mathbf{\Psi}$:
\begin{equation}
\label{eq:diff_fixedpt_cond}
    \frac{\partial \mathbf{x}^{\star}}{\partial \mathbf{c}}(\mathbf{c}) = 
    \underbrace{{\frac{\partial \mathcal{U} }{\partial \mathbf{x}^{\star}} (\mathbf{x}^{\star}(\mathbf{c}), \; \mathbf{c} )}}_{\mathbf{\Phi}}
    	\cdot     
    \frac{\partial\mathbf{x}^{\star}}{\partial \mathbf{c}}(\mathbf{c}) + 
    \underbrace{\frac{\partial \mathcal{U} }{\partial \mathbf{c}} (\mathbf{x}^{\star}(\mathbf{c}), \; \mathbf{c} )}_{\mathbf{\Psi}}  
    ,
\end{equation}
by the chain rule and recognizing the implicit and explicit dependence of $\mathcal{U}$ on the independent parameters $\mathbf{c}$. Equation (\ref{eq:diff_fixedpt_cond}) will be called the \emph{differential fixed-point conditions}. 
Rearranging (\ref{eq:diff_fixedpt_cond}), the desired $\frac{\partial \mathbf{x}^{\star}}{\partial \mathbf{c}} (\mathbf{c})$ can be found in terms of $\mathbf{\Phi}$ and $\mathbf{\Psi}$ as defined above, to yield the system (\ref{eq:Lemma_fixedpt}) below.

The results discussed next are valid under the assumptions that $\mathbf{x}^{\star}\!\!:\!\mathbb{R}^n \!\to \!\mathbb{R}^n$ is differentiable in an open set $\mathcal{C}$, and Equation (\ref{eq:fixedpt_cond}) holds for $\mathbf{c} \in \mathcal{C}$. Additionally, $\mathcal{U}$ is assumed differentiable on an open set containing the point $(\mathbf{x}^{\star}( \mathbf{c} ), \mathbf{c})$. 

\begin{lemma}
\label{lemma:fixedpt-diff}
When $\mathbf{I}$ is the identity operator  and $\mathbf{\Phi}$ nonsingular, 
\begin{equation}
    \label{eq:Lemma_fixedpt}  \tag{DFP}
        (\mathbf{I} - \mathbf{\Phi} ) \frac{\partial \mathbf{x}^{\star}}{\partial \mathbf{c}} = \mathbf{\Psi} .
\end{equation}
\end{lemma}

The result follows from the Implicit Function theorem \citep{munkres2018analysis}. It implies that the Jacobian $\frac{\partial \mathbf{x}^{\star}}{\partial \mathbf{c}}$ can be found as the solution to a linear system once the prerequisite Jacobians $\mathbf{\Phi}$ and $\mathbf{\Psi}$ are found; these Jacobians correspond to backpropagation through the update function $\mathcal{U}$ at $\mathbf{x}^{\star}(\mathbf{c})$, with respect to $\mathbf{x}^{\star}$ and $\mathbf{c}$.

Using the above two Lemmas, the central result of the paper can be proved. Informally, it states that the backward pass of an iterative solver (\ref{eq:opt_iteration}), unfolded at a precomputed optimal solution $\mathbf{x}^{\star}(\mathbf{c})$, is equivalent to solving the linear equations (\ref{eq:Lemma_fixedpt}) using linear fixed-point iteration, as outlined in Lemma \ref{lemma:linear-iteration}. This perspective allows insight into the convergence properties of this backpropagation, including its convergence rate, and shows that more efficient algorithms can be used to solve (\ref{eq:Lemma_fixedpt}) in favor of its inherent LFPI implementation in unfolding.

The following results hold under the assumptions that 
the parameterized optimization mapping $\mathbf{x}^{\star}$ converges for certain parameters $\mathbf{c}$ through a sequence of iterates $\mathbf{x}_k(\mathbf{c}) \to \mathbf{x}^{\star}(\mathbf{c})$ using algorithm (\ref{eq:opt_iteration}), 
and that $\mathbf{\Phi}$ is nonsingular with a spectral radius $\rho(\mathbf{\Phi}) < 1$.

\begin{theorem}
\label{thm:unfolding_convergence_fixedpt}
The backward pass of an unfolding of algorithm (\ref{eq:opt_iteration}), starting at the point $\mathbf{x}_{k} = \mathbf{x}^{\star}$, is equivalent to linear fixed-point iteration on the linear system (\ref{eq:Lemma_fixedpt}), and will converge to its unique solution   at an asymptotic rate of
\begin{equation}
    - \log \rho(\mathbf{\Phi}).
\end{equation}

\end{theorem}

\begin{proof}
Since $\mathcal{U}$ converges given any parameters $\mathbf{c} \in \mathcal{C}$,  Equation (\ref{eq:fixedpt_cond}) holds for any $\mathbf{c} \in \mathcal{C}$. Together with the assumption the $\mathcal{U}$ is differentiable on a neighborhood of $(\mathbf{x}^{\star}( \mathbf{c} ),\; \mathbf{c})$,
\begin{equation}
        (\mathbf{I} - \mathbf{\Phi} ) \frac{\partial \mathbf{x}^{\star}}{\partial \mathbf{c}} = \mathbf{\Psi}
\end{equation}
holds by Lemma \ref{lemma:fixedpt-diff}. 
When (\ref{eq:opt_iteration}) is unfolded, its backpropagation rule can be derived by differentiating its update rule:
\begin{subequations}
\begin{align}
        \label{eq:diff_update_1}
        \frac{\partial}{\partial \mathbf{c}} \left[ \; \mathbf{x}_{k+1}(\mathbf{c}) \;\right]  &=   \frac{\partial }{\partial \mathbf{c}}  \left[ \; \mathcal{U}(\mathbf{x}_k (\mathbf{c}) ,\;  \mathbf{c} )   \;\right] \\
        \label{eq:diff_update_2}
        \frac{\partial \mathbf{x}_{k+1}}{\partial \mathbf{c}}\; (\mathbf{c}) &=          \frac{\partial \mathcal{U}}{\partial \mathbf{x}_k} \frac{\partial \mathbf{x}_k}{\partial \mathbf{c}} + 
        \frac{\partial \mathcal{U}}{\partial \mathbf{c}},
\end{align}
\end{subequations}
where all terms on the right-hand side are evaluated at $\mathbf{c}$ and $\mathbf{x}_k (\mathbf{c})$. Note that in the base case $k=0$, since in general $\mathbf{x}_{0}$ is arbitrary and does not depend on $\mathbf{c}$, $\frac{\partial \mathbf{x}_0}{\partial \mathbf{c}} = \mathbf{0}$ and  
\begin{equation}
\label{eq:initial-update-diff}
        \frac{ \partial \mathbf{x}_{1}}{\partial \mathbf{c}}(\mathbf{c}) =
        \frac{\partial \mathcal{U}}{\partial \mathbf{c}}(\mathbf{x}_{0}, \mathbf{c}).
\end{equation}
This holds also when $\mathbf{x}_0 \!=\! \mathbf{x}^{\star}$ w.r.t.~backpropagation of (\ref{eq:opt_iteration}), since $\mathbf{x}^{\star}$ is precomputed outside the computational graph of its unfolding.  Now since $ \mathbf{x}^{\star}$ is a fixed point of (\ref{eq:opt_iteration}), 
\begin{equation}
\label{eq:xstar_all_k}
        \mathbf{x}_k(\mathbf{c}) = \mathbf{x}^{\star}(\mathbf{c})  \;\;\; \forall k \geq 0,
\end{equation}
which implies 
\begin{subequations}
\begin{align}
\label{eq:du_dxstar_all_k}
\frac{\partial \mathcal{U}}{\partial \mathbf{x}_k}(\mathbf{x}_k( \mathbf{c}), \; \mathbf{c} ) &= \frac{\partial \mathcal{U}}{\partial \mathbf{x}^{\star}}(\mathbf{x}^{\star}( \mathbf{c}), \; \mathbf{c} ) = \mathbf{\Phi}, \;\;\; \forall k \geq 0\\ 
\label{eq:du_dxstar_all_k2}
\frac{\partial \mathcal{U}}{\partial \mathbf{c}}(\mathbf{x}_k( \mathbf{c}), \; \mathbf{c} ) &= \frac{\partial \mathcal{U}}{\partial \mathbf{c}}(\mathbf{x}^{\star}( \mathbf{c}), \; \mathbf{c} ) = \mathbf{\Psi}, \;\;\; \forall k \geq 0.
\end{align}
\end{subequations}
Letting $\mathbf{J}_{k} \!\coloneqq\! \frac{\partial \mathbf{x}_k}{\partial \mathbf{c}}(\mathbf{c})$, the rule (\ref{eq:diff_update_2}) for unfolding at a fixed-point $\mathbf{x}^{\star}$ becomes, along with initial conditions (\ref{eq:initial-update-diff}), 
\begin{subequations}
\label{eq:proof_backprop}
\begin{align}
        \mathbf{J}_{0} &=  \mathbf{\Psi}\\
        \mathbf{J}_{k+1} & =
         \mathbf{\Phi} \mathbf{J}_{k} + 
        \mathbf{\Psi}.
\end{align}    
\end{subequations}
The result then holds by direct application of Lemma \ref{lemma:linear-iteration} to (\ref{eq:proof_backprop}), recognizing $\mathbf{z}_k = \mathbf{J}_{k}$ , $\mathbf{B} = \mathbf{\Phi}$  and  $\mathbf{z}_0 = \mathbf{b} = \mathbf{\Psi}$.
\end{proof}

\noindent The following is a direct result from the proof of Theorem~\ref{thm:unfolding_convergence_fixedpt}.
\begin{corollary}
Backpropagation of the fixed-point unfolding consists of the following rule:
\begin{subequations}
\label{eq:update-diff-fixed}
\begin{align}
        \mathbf{J}_{0} &=  \mathbf{\Psi}\\
        \mathbf{J}_{k+1} & =
         \mathbf{\Phi} \mathbf{J}_{k} + 
        \mathbf{\Psi},
\end{align}    
\end{subequations}
where  $\mathbf{J}_{k} \coloneqq \frac{\partial \mathbf{x}_k}{\partial \mathbf{c}}(\mathbf{c})$.
\end{corollary}


Theorem \ref{thm:unfolding_convergence_fixedpt} specifically applies to the case where the initial iterate is the precomputed optimal solution, $\mathbf{x}_0 = \mathbf{x}^{\star}$. However, it also has implications for the general case where $\mathbf{x}_0$ is arbitrary. As the forward pass optimization converges, i.e. $\mathbf{x}_k \to \mathbf{x}^{\star}$ as $k \to \infty$, this case becomes identical to the one proved in Theorem \ref{thm:unfolding_convergence_fixedpt} and a similar asymptotic convergence result applies. If $\mathbf{x}_k \to \mathbf{x}^{\star}$ and $\mathbf{\Phi}$ is a nonsingular operator with $\rho(\mathbf{\Phi}) < 1$, the following result holds.
\begin{corollary}
\label{corr-unfolding-convergence}
When the parametric problem (\ref{eq:opt_generic}) can be solved by an iterative method of the form (\ref{eq:opt_iteration}) and the forward pass of the unfolded algorithm converges, the backward pass converges at an asymptotic rate that is bounded by $-\log \rho(\mathbf{\Phi})$.
\end{corollary}
The above results can help to explain the empirical convergence patterns of the illustrative example in the beginning of this Section.  First, they help explain  the difference in  forward and backward-pass convergence rates due to unfolded PGD  as shown in row $1$ of Figure \ref{fig:fwd_bwd_err}. Regardless of the convergence rate of its forward pass solution, the overall convergence rate of an unfolded optimization is limited by that of the LFPI implicity applied in its backward pass. It is also clear why the backward pass of unfolding converges faster when its forward pass is initialized at the optimal solution $\mathbf{x}_0 = \mathbf{x}^{\star}$: the correct $\mathbf{\Phi}$ and $\mathbf{\Psi}$ are exactly known at every iteration in this case, and backpropagation follows the rule \eqref{eq:update-diff-fixed}. In the typical case when $\mathbf{x}_0$ is chosen randomly, $\mathbf{\Phi}$ and $\mathbf{\Psi}$ are "moving targets" with respect to the LFPI iterations \ref{eq:update-diff-fixed}, so the convergence of LFPI in this case is bound to lag behind the convergence of $\mathbf{x}_k$ to $\mathbf{x}^{\star}$.

\section{Folded Optimization}
\label{sec:FoldedOptimization}

As noted in Section \ref{sec:illustrative_example}, the fixed-point unfolding approach to backpropagation is inefficient because it requires precomputation of $\mathbf{x}^{\star}$ along with additional unfolded iterations \eqref{eq:opt_iteration}. In principle, this inefficiency can be addressed by noting that the forward pass of each unfolded iteration \eqref{eq:opt_iteration} need not be recomputed at the fixed point $\mathbf{x}^{\star}$, since $\mathbf{x}_k(\mathbf{c}) = \mathbf{x}^{\star}(\mathbf{c})$ $\forall k\geq 0$ and $ \mathcal{U}( \mathbf{x}^{\star}(\mathbf{c}), \mathbf{c} ) = \mathbf{x}^{\star}(\mathbf{c})$. Thus we can iterate just its backward pass \eqref{eq:update-diff-fixed} repeatedly at $\mathbf{x}^{\star}(\mathbf{c})$, for which the requisite Jacobians $\mathbf{\Phi}$ and $\mathbf{\Psi}$ can be obtained by a just a \emph{single} application of the differentiable update step $\mathcal{U}$. This revised approach leads to the most basic variant of  \emph{fixed-point folding}. 

The essence of fixed-point folding is to use the computational graph of the update step $\mathcal{U}$ to backpropagate the function $\mathbf{c} \rightarrow \mathbf{x}^{\star}(\mathbf{c})$ by modeling and solving the linear system \eqref{eq:Lemma_fixedpt}, after the optimal solution $\mathbf{x}^{\star}(\mathbf{c})$ is separately furnished by any \emph{blackbox} optimization solver. This is in contrast to unrolled or unfolded optimization, which jointly solves for the optimal solution and its backpropagated gradients by repeated application of $\mathcal{U}$ with automatic differentiation enabled. The separation of the forward and backward pass algorithms, which are typically entangled across iterations of unfolding, is key to enabling several practical advantages as detailed below.

This Section describes a system called \emph{folded optimization}, which encompasses a variety of implementation strategies for fixed-point folding. The paper is also accompanied by an open-source PyTorch library called \texttt{fold-opt}, which provides practical implementations of folded optimization variants within a convenient user interface. Its function is to facilitate the conversion of unfolded optimization code into more efficient and reliable JgP-based differentiable optimization. To produce a differentiable mapping $\mathbf{c} \rightarrow \mathbf{x}^{\star}(\mathbf{c})$ in \texttt{fold-opt}, two elements are required: a differentiable step $\mathcal{U}$ of an iterative method which solves the problem \eqref{eq:opt_generic}, along with any (blackbox) optimization oracle which provides optimal solutions $\mathbf{x}^{\star}(\mathbf{c})$ given $\mathbf{c}$. Note that both of these elements are always available given any unfolded implementation of \eqref{eq:opt_iteration}: the former is equivalent to setting the number of unfolded iterations to one. The backpropagation algorithms employed by \texttt{fold-opt} are proposed next.

\subsection{Folded Optimization: Algorithms}

Given access to an optimization solver $\mathbf{c} \rightarrow \mathbf{x}^{\star}(\mathbf{c})$ and differentiable update step $\mathcal{U}$, the goal is to compute a  JgP mapping $\mathbf{g} \rightarrow \mathbf{g}^T  \mathbf{J}$ where  $\mathbf{g} = \frac{\partial \mathcal{L}}{\partial \mathbf{x}^{\star}}$ is the incoming gradient and the matrix $\mathbf{J} = \frac{\partial \mathbf{x}^{\star}(\mathbf{c}) }{\partial \mathbf{c} }$ solves the linear system $(\mathbf{I} - \mathbf{\Phi} ) \mathbf{J} = \mathbf{\Psi} $; thus $\mathbf{g}^T \mathbf{J} = \frac{\partial \mathcal{L}}{\partial \mathbf{c}}$. While the Jacobian matrices $\mathbf{\Phi}$, $\mathbf{\Psi}$ and $\mathbf{J}$ are not known explicitly, the products $\mathbf{g}^T \mathbf{\Phi}$ and $\mathbf{g}^T \mathbf{\Psi}$ can be computed by backpropagation of any vector $\mathbf{g}$ through the computational graph of $\mathcal{U}(\mathbf{x}^{\star}(\mathbf{c}),(\mathbf{c}))$ backward to $\mathbf{x}^{\star}(\mathbf{c})$ and $\mathbf{c}$, respectively. Thus, the backpropagation algorithms of \texttt{fold-opt} are designed to compute the desired mapping $\mathbf{g} \rightarrow \mathbf{g}^T  \mathbf{J}$, by using the available mappings $\mathbf{g} \rightarrow \mathbf{g}^T  \mathbf{\Phi}$ and $\mathbf{g} \rightarrow \mathbf{g}^T  \mathbf{\Psi}$, which can be obtained by calling $\mathcal{U}$ only once and saving its computational graph. The library implements three distinct approaches to this end, detailed next.

\subsubsection{Linear Fixed-Point Iteration}
The first variant of folded optimization mimics unfolding at the fixed point $\mathbf{x}^{\star}$ by solving a linear system for the product $\mathbf{g}^T \mathbf{J}$, using a variation of the LFPI algorithm \eqref{eq:linear-iteration}. By construction, it is algorithmically equivalent to the backpropagation of fixed-point unfolding \eqref{eq:update-diff-fixed}. 

To see how, write the backpropagation of the loss gradient $\frac{\partial \mathcal{L}}{\partial \mathbf{x}^{\star}}$ through $k$ unfolded steps of \eqref{eq:opt_iteration} at the fixed point $\mathbf{x}^{\star}$ as 
\begin{equation}
\label{eq:backprop_goal}
\frac{\partial \mathcal{L}}{\partial \mathbf{x}^{\star}}^T \left(   \frac{\partial \mathbf{x}^{k}(\mathbf{c})}{\partial \mathbf{c}}  \right) .
\end{equation}
We seek to compute the limit $\frac{\partial \mathcal{L}}{\partial \mathbf{c}} = \mathbf{g}^T \mathbf{J}$ where $\mathbf{g} = \frac{\partial \mathcal{L}}{\partial \mathbf{x}^{\star}}$, $\mathbf{J} \coloneqq \lim_{k \rightarrow \infty} \mathbf{J}_k  $  , and  $\mathbf{J}_k =  \frac{\partial \mathbf{x}^{k}(\mathbf{c})}{\partial \mathbf{c}} $  . Following the backpropagation rule \eqref{eq:update-diff-fixed}, the expression \eqref{eq:backprop_goal} is equal to 
\begin{subequations}
\begin{align}
    \mathbf{g}^T \mathbf{J}_k &=  \mathbf{g}^T \left( \mathbf{\Phi} \mathbf{J}_{k-1} + \mathbf{\Psi} \right) \\
    &=  \mathbf{g}^T \left( \mathbf{\Phi}^k \mathbf{\Psi} + \mathbf{\Phi}^{k-1} \mathbf{\Psi}  + \ldots + \mathbf{\Phi} \mathbf{\Psi}  + \mathbf{\Psi}   \right)
\end{align}
\end{subequations}

This expression can be rearranged as
\begin{equation}
\label{eq:rearrange}
    \mathbf{g}^T \mathbf{J}_k  =   \mathbf{v}_k^T \mathbf{\Psi} 
\end{equation}

\noindent where
\begin{equation}
\label{eq:def_v}
 \mathbf{v}_k^T \coloneqq \left(\mathbf{g}^T  \mathbf{\Phi}^k  + \mathbf{g}^T \mathbf{\Phi}^{k-1}   + \ldots + \mathbf{g}^T  \mathbf{\Phi}   + \mathbf{g}^T    \right)   .
\end{equation}
The sequence $\mathbf{v}_k$ can be computed most efficiently as
\begin{equation}
\label{eq:recursion_v}
  \mathbf{v}_k^T = \mathbf{v}_{k-1}^T \mathbf{\Phi} + \mathbf{g}^T  
\end{equation}
which identifies $ \mathbf{v} \coloneqq  \lim_{k \rightarrow \infty} \mathbf{v}_k  $ as the solution of the linear system
\begin{equation}
\label{eq:v_system}
  \mathbf{v}^T (\mathbf{I} - \mathbf{\Phi})   =   \mathbf{g}^T  
\end{equation}
under the conditions of Lemma \eqref{lemma:linear-iteration}, after transposing both sides of \eqref{eq:recursion_v} and \eqref{eq:v_system} . 

Once $\mathbf{v}^T$ is calculated by \eqref{eq:recursion_v}, the desired JgP is
\begin{equation}
\label{eq:gTJ}
  \mathbf{g}^T \mathbf{J} = \mathbf{v}^T \mathbf{\Psi}  .
\end{equation}

Thus the end result $\mathbf{g}^T \mathbf{J}$ is computed by iterating \eqref{eq:recursion_v} to find $\mathbf{v}$ which solves \eqref{eq:v_system}, and then applying \eqref{eq:gTJ}. The left matrix-vector product with respect to  $\mathbf{\Phi}$ in \eqref{eq:recursion_v} and $\mathbf{\Psi}$ in \eqref{eq:gTJ} can be computed by backpropagation through the computational graph of the update function $\mathcal{U}(\mathbf{x}^{\star}(\mathbf{c}), \; \mathbf{c})$, backward  to $\mathbf{x}^{\star}(\mathbf{c})$ and $\mathbf{c}$ respectively. Notice that in contrast to unfolding, this backpropagation method requires to store the computational graph only for a single iteration of the update step, rather than for the entire optimization routine consisting of many iterations. 

One remaining detail is to initialize the iterates \eqref{eq:recursion_v} by choosing $\mathbf{v}_0$. The choice of $\mathbf{v}_0$ does not affect aymptotic convergence of \eqref{eq:recursion_v}. However to make LFPI in \texttt{fold-opt} completely equivalent with the backpropagation of fixed-point unfolding, Equation \eqref{eq:initial-update-diff} shows that the initial iterate must be chosen as follows:

\begin{equation}
\label{eq:v0}
  \mathbf{v}_0 = \mathbf{g}^T \mathbf{\Psi} .
\end{equation}

\paragraph{\bf Empirical Illustration} The third row of Figure \ref{fig:fwd_bwd_err} shows the convergence pattern of backpropagation using \emph{fold-opt} in LFPI mode, on the illustrative example of Section \ref{sec:Unfolding_at_a_fixed_point}. As intended, its error curves are nearly identical to those of the second row, which result from unfolding projected gradient descent at its fixed point. Minute differences between the second and third rows of curves can be attributed to numerical floating-point error. 

It is apparent from Figure \ref{fig:fwd_bwd_err} that due to its equivalence with fixed-point unfolding, backpropagation by \texttt{fold-opt} LFPI inherits a dependence on the optimization parameter used to define $\mathcal{U}$, in this case the gradient descent stepsize $\alpha$. This can be explained by Theorem \ref{thm:unfolding_convergence_fixedpt}: the stepsize $\alpha$ affects $\mathcal{U}$ and thus $\mathbf{\Phi}$, along with its spectral radius. In turn, this determines the asymptotic convergence rate of backpropagation by LFPI. 

Figure \ref{fig:fixed_point_conv_3D_LFPI} gives an expanded view of this aspect, showing the relationship between stepsize $\alpha$, the spectral radius $\rho$, and error per iteration due to LFPI backpropagation in \texttt{fold-opt}. Note that Figure \ref{fig:fixed_point_conv_3D_LFPI} coincides with the third row of Figure \ref{fig:fwd_bwd_err}, for its four values of $\alpha$. Note the continuous relationship between $\alpha$ and the backward-pass convergence rate, which finds a global maximum within the range of alpha shown. The spectral radius $\rho(\Phi)$ is represented in color coding; as predicted by Theorem \ref{thm:unfolding_convergence_fixedpt}, it is minimized precisely where the convergence rate is maximized. Further, note that the backward pass fails to converge precisely as $\rho(\mathbf{\Phi})$ exceeds the value $1$: that is, the first black curve does not intercept the xy-plane. Taken together, these observations corroborate and provide empirical evidence for the main implications of Theorem \ref{thm:unfolding_convergence_fixedpt}.

\begin{figure}
    \centering
    \includegraphics[width=0.8\linewidth]{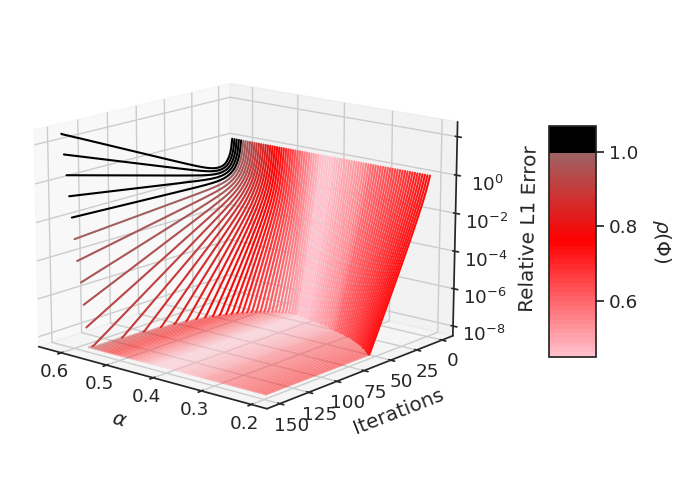} 

    \caption{An expanded view of Figure \ref{fig:fwd_bwd_err}'s third row shows backward-pass convergence for fixed-point folding of PGD by GMRes, compared to stepsize $\alpha$ and spectral radius  $\mathbf{\Phi}$ (color scale) on CIFAR100 Multilabel Classification. The main consequences of Theorem \ref{thm:unfolding_convergence_fixedpt} are illustrated: convergence rate is maximized when the spectral radius of $\mathbf{\Phi}$ is minimized, and failure to converge coincides with when the spectral radius exceeds $1$.}
    \label{fig:fixed_point_conv_3D_LFPI}
\end{figure}

\subsubsection{Krylov Subspace Methods}
\label{sec:GMRES}

The main drawback of backpropagation by LFPI, as described above, is its slow convergence rate. As a generic linear system solver, fixed-point iteration is typically not used in practice due to availability of faster-converging variants including Jacobi, Gauss-Seidel, Successive Over-Relaxation, and Krylov subspace methods \cite{quarteroni2010numerical}. The main advantage of LFPI, which justifies its use in the present context, is its compatibility with solution of the linear system \eqref{eq:Lemma_fixedpt} by matrix-vector products, since multiplication by $\mathbf{\Phi}$ and $\mathbf{\Psi}$ coincide with backpropagation through $\mathcal{U}$. This allows for the solution of \eqref{eq:Lemma_fixedpt} without explicitly computing $\mathbf{\Phi}$ and $\mathbf{\Psi}$.

Another, more efficient class of linear system solvers which share this characteristic are the Krylov subspace methods \cite{quarteroni2010numerical}. In order to solve a generic linear system $\mathbf{A} \mathbf{x} = \mathbf{b}$, these methods generally act on the basis vectors of a \emph{Krylov subspace}  generated by $A$ and a vector $r$:

\begin{equation}
\label{eq:Krylov_subspace}
  \mathcal{K}_k(\mathbf{A};\mathbf{r}) = \textit{span} \left( \{ \mathbf{r},\; \mathbf{A} \mathbf{r},\; \mathbf{A}^2 \mathbf{r},\; \ldots ,\; \mathbf{A}^{k-1} \mathbf{r}    \} \right),
\end{equation}

\noindent but otherwise do not require $\mathbf{A}$ explicitly. Therefore, they can be implemented in the present context to solve \eqref{eq:Lemma_fixedpt}, where only the matrix-vector products with respect to $\mathbf{\Phi}$ and $\mathbf{\Psi}$ are known. 

GMRes is the most popular of the general-purpose Krylov subspace methods. It solves at each $k^{th}$ iteration for the minimal-residual vector $\argmin_x \| \mathbf{A} \mathbf{x} - \mathbf{b} \|^2$ which lies within the $k^{th}$ Krylov subspace $\mathcal{K}_k(\mathbf{A};\mathbf{r})$, where $\mathbf{r} = \mathbf{A} \mathbf{x}_0 - \mathbf{b}$ is the residual of an initial guess. Arnoldi iteration is used to incrementally generate orthonormal bases for the Krylov subspaces, over which the minimial-residual vector can be computed by least-squares via QR decomposition. The decomposition is efficiently updated in order to resolve the least-squares problem at each iteration \cite{sauer2011numerical}.

The  \texttt{fold-opt} library implements a variant of GMRes to perform backpropagation as an alternative to LFPI. To see how, note that it was shown above how fixed-point unfolding can be interpreted as solving the equation \eqref{eq:v_system} for $\mathbf{v}$ by LFPI and then applying \eqref{eq:gTJ}. Here we follow the same pattern,  solving \eqref{eq:v_system} instead by GMRes. Since the algorithm is developed for left-sided linear systems, we transpose \eqref{eq:v_system} to make its orientation consistent with the typical GMRes formulation: 

\begin{equation}
\label{eq:v_system_T}
   (\mathbf{I} - \mathbf{\Phi})^T \mathbf{v}  =   \mathbf{g}  
\end{equation}

\noindent where $\mathbf{A} = (\mathbf{I} - \mathbf{\Phi})^T$. Now the Krylov subspace basis vectors can be computed as via multiplication of a vector $\mathbf{r}$ by  $(\mathbf{I} - \mathbf{\Phi})^T$ as follows:

\begin{equation}
\label{eq:krylov_mult}
   (\mathbf{I} - \mathbf{\Phi})^T \mathbf{r}  =   (\mathbf{r} - \mathbf{r}^T  \mathbf{\Phi})^T 
\end{equation}

where $\mathbf{r}^T  \mathbf{\Phi}$ is, once again, computed by backpropagation of $\mathbf{r}$ through $\mathcal{U}$ to $\mathbf{x}^{\star}$. The rest of the method follows a  conventional implementation.

The GMRes method typically converges in far fewer iterations than LFPI. This benefit comes at an additional cost of $\mathcal{O}(km)$ flops per each $k^{th}$ iteration, where $m$ is the size of $\mathbf{\Phi}$. Additionally, exact convergence of the backward pass is guaranteed within $m$ iterations \cite{sauer2011numerical}, at which point the Krylov subspace coincides with $\mathbb{R}^m$. This result is significant, because it highlights the inferior backward-pass convergence properties inherent to unrolling and unfolding optimization, which are limited to the $- \log \rho(\mathbf{\Phi})$ convergence rate of LFPI, by producing equivalent results with faster convergence.  

\paragraph{\bf Empirical Illustration}

The fourth and final row of Figure \ref{fig:fwd_bwd_err} shows the backward-pass error per iteration due to fixed-point folding with GMRes as described above. As expected, the backward pass converges in far fewer iterations when compared to the other backpropagation rules. Figure \ref{fig:fixed_point_conv_3D_GMRES} shows a more complete view of the relationship between PGD stepsize, spectral radius and convergence pattern. It can be compared directly to Figure \ref{fig:fixed_point_conv_3D_LFPI}. Note that backpropagation with GMRes is not subject to the same effect on convergence pattern due to changes in stepsize as in LFPI. This is because convergence of GMRes does not depend on $\mathbf{\Phi}$ being a contractive mapping with small spectral radius. Note in particular how convergence is reached in few iterations even when $\rho(\mathbf{\Phi}) > 1$ (in black). This is a significant improvement over unfolding and fixed-point folding with LFPI, since their backward-pass convergence rate depends on the spectral radius, whose relationship to optimization parameters such $\alpha$ can be difficult to calculate.

\begin{figure}
    \centering
    \includegraphics[width=0.7\linewidth]{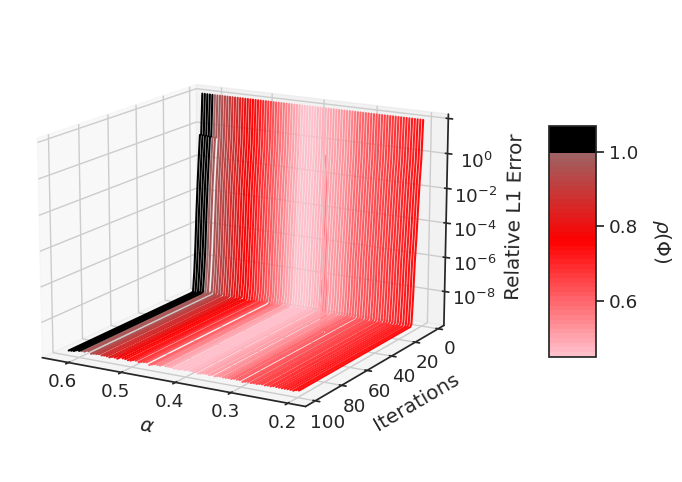} 
    \caption{An expanded view of Figure \ref{fig:fwd_bwd_err}'s third row shows backward-pass convergence for fixed-point folding of PGD by GMRes, compared to stepsize $\alpha$ and spectral radius  $\mathbf{\Phi}$ (color scale) on CIFAR100 Multilabel Classification. Because GMRes does not depend on iterating a contractive mapping with low spectral radius, convergence rates are unaffected by the stepsize of PGD used to backpropagate gradients.}
    \label{fig:fixed_point_conv_3D_GMRES}
\end{figure}

\subsubsection{Jacobian Extraction} A final alternative method to solve for the backpropagated gradients $\mathbf{g}^T \mathbf{J}$ is to solve the system \eqref{eq:v_system} directly, by first building the Jacobian matrix $\mathbf{\Phi}$. This is done by back-propagating the identity matrix through $\mathcal{U}$ backward to $\mathbf{x}^{\star}$. Subsequently, \eqref{eq:gTJ} is applied using backpropagation through $\mathcal{U}$ to $\mathbf{c}$. In this approach, any linear equation solver can be applied since the system \eqref{eq:v_system} is known explicitly. The \texttt{fold-opt} library allows the user to pass a blackbox linear system solver when this option is chosen.

The downside to this approach is the cost of building the matrix $\mathbf{\Phi}$. When $\mathbf{\Phi} \in \mathbb{R}^{m \times m}$, this requires backpropagating each of the $m$ columns of the identity matrix in addtion to the cost of solving \eqref{eq:v_system}. For comparison, backpropagation by GMRes as described in \ref{sec:GMRES} is guaranteed to reach full convergece within the same number iterations, each requiring one backward pass through $\mathcal{U}$.

\subsection{Folded Optimization: Practical Considerations}

The Section is concluded with a discussion of some practical aspects when using folded optimization. Here, emphasis is given to the potential pitfalls of unrolling and unfolding optimization, which are addressed in the folded optimization system.

\paragraph{Blackbox Optimization} One of the primary benefits of folded optimization is the ability to leverage blackbox optimization solvers to compute the forward-pass mapping $\mathbf{c} \rightarrow \mathbf{x}^{\star}(\mathbf{c})$. The ability to accomodate blackbox solvers is an important efficiency advantage that is precluded by unrolled optimization, since it requires the solver to be implemented in an AD environment. Most practical applications of optimization rely on highly optimized software implementations such as Gurobi \citep{gurobi}, which can incorporate problem-specific handcrafted heuristics as well as low-level code optimizations to minimize solving time. This is also a major advantage over the existing differentiable optimization library \texttt{cvxpy}, which requires converting the problem to a convex cone program before solving it with a specialized operator-splitting method for conic programming \citep{agrawal2019differentiable}, rendering it inefficient for many optimization problems.

\paragraph{Parameter Selection} Optimization methods typically require specification of parameters such as gradient stepsizes, which can be chosen as constants or adaptively at each iteration. Even when such parameters can be well-chosen for forward-pass convergence, the same values may not perform well for backward-pass convergence in unfolded optimization. This potential hazard of unfolding is illustrated in Figure \ref{fig:polyaks}, which again shares the illustrative example of PGD on multiclass selection. Here Polyak's adaptive stepsize rule is used, which guarantees convergence of PGD \cite{beck2017first}. However, since Polyak's rule decays the stepsize to zero, convergence of the backward pass slows over time, causing it to flatline at four orders of magnitude in error behind the forward pass. The equivalent result due to a constant stepsize (in dotted curves) serves to show how a constant, finite stepsize leads to much more efficient backward-pass convergence.  This highlights the importance of separating the forward and backward-pass models in fixed-point folding, so that convergence of both passes can be ensured.

\begin{figure}
    \centering
    \includegraphics[width=0.7\linewidth]{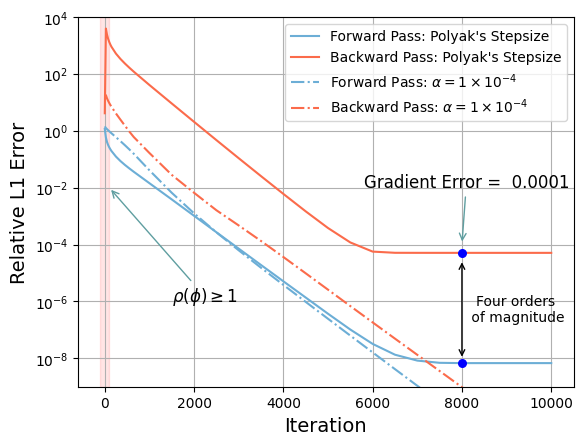} 
    \caption{Impact of Polyak's Stepsize Rule on forward and backward convergence when unfolding PGD on CIFAR100 Multilabel Classification. Convergence in the forward pass is guaranteed, but relies on decaying the stepsize asymptotically to zero, which causes failure to converge in the backward pass.}
    \label{fig:polyaks}
\end{figure}

\paragraph{Monitoring Backward Convergence} Error tolerance thresholds are often used to terminate optimization methods when sufficient accuracy is reached. In typical algorithm unrolling, it is not possible to monitor the   backward pass for termination by early stopping, since it is fully determined by the forward pass. 

\paragraph{Computational Graph Size} Since folded optimization requires the computational graph of $\mathcal{U}$ for only a single iteration at the fixed point $\mathbf{x}^{\star}$, its time and space efficiency are potentially much higher than that of the typical unrolled optimization, which stores a chain of computational graphs through $\mathcal{U}$ from an arbitrary starting point to the optimal solution at convergence.

\paragraph{Nested Fixed-Point Folding}
As noted in Section \ref{sec:unfolding}, an optimization method may require in its update step \ref{eq:opt_iteration} the solution of an optimization subproblem; when the subproblem itself requires an iterative method, this leads to nested unrolled loops. It is important to note that as per Definition \ref{def:unfolding}, the innermost optimization loop of a nested unrolling can be considered an unfolding and can be converted to fixed-point folding using the methods of this section. Subsequently, the next outermost loop can now be considered unfolded, and the same process applied until all unrolled loops are replaced solution of their respective analytical models. The process is exemplified by \textit{f-PGDb} (introduced in Section \ref{sec:Experiments}), which applies successive fixed-point folding through ADMM for and PGD (described in Section \ref{sec:unfolding}) to compose a JgP-based differentiable layer for any optimization problem with a smooth objective function and linear constraints. In particular, quadratic programming by ADMM is used to define $\mathcal{U}$ for a differentiable projection onto linear constraints, resulting in unfolded PGD. Then, fixed-point folding is applied again to replace the unfolded PGD loop.  This \textit{f-PGDb} module is used to backpropagate nonconvex quadratic programming, and neither ADMM nor PGD is used to compute the optimal solution in the forward pass. For this, Gurobi solver is used as a black box optimization oracle.

\section{Experiments}
\label{sec:Experiments}

This Section evaluates folded optimization on five different end-to-end optimization and learning tasks. It is primarily evaluated against \texttt{cvxpy}, which is the preeminent general-purpose differentiable optimization solver. Two crucial limitations of \texttt{cvxpy} are its efficiency and expressiveness. This is due to its reliance on transforming general optimization programs to convex cone programs, before applying a standardized operator-splitting cone program solver and differentiation scheme (see Section \ref{sec:RelatedWork}). This precludes the incorporation of problem-specific solvers in the forward pass  and limits its use to convex problems only. 
One major benefit of \texttt{fold-opt} is the modularity of its forward optimization pass, which can apply any blackbox algorithm to produce $\mathbf{x}^{\star}(\mathbf{c})$. In each experiment below, this is used to demonstrate a different advantage.

The experiments test four differentiable optimation modules implemented in \texttt{fold-opt}: 
\textbf{(1)} \textit{f-PGDa} applies to optimization mappings with linear constraints, and is based on folding projected gradient descent steps, 
where each inner projection is a QP solved by the differentiable QP solver \texttt{qpth} \citep{amos2019optnet}. 
\textbf{(2)} \textit{f-PGDb} is a variation on the former, in which the inner QP step is differentiated by fixed-point folding of the ADMM solver specified in \eqref{app:QP_ADMM}. 
\textbf{(3)} \textit{f-SQP} applies to optimization with nonlinear constraints and uses folded SQP with the inner QP differentiated by \texttt{qpth}. 
\textbf{(4)} \textit{f-FDPG} comes from fixed-point folding of the Fast Dual Proximal Gradient Descent (FDPG) shown in  \ref{appendix:models}. 
Its inner optimization step \eqref{eq:inner_optimization} is a soft thresholding $\operatorname{Prox}$ operator, whose simple closed form is differentiated by AD in PyTorch.

\paragraph{Decision-focused Learning Setting} The first three tasks of this Section follow the problem setting known as \emph{Decision-focused Learning}, or \emph{Predict-Then-Optimize}. Here, an optimization problem \eqref{eq:opt_generic} has unknown coefficients only in its objective function $f(\mathbf{x},\mathbf{c})$ while the constaints are considered constant. The goal of the supervised learning task is to predict $\hat{\mathbf{c}}$ from feature data such that the resulting $\mathbf{x}^{\star}(\hat{\mathbf{c}})$  optimizes the objective under ground-truth parameters $\bar{\mathbf{c}}$, which is $f(\mathbf{x}^{\star}(\hat{\mathbf{c}}),\bar{\mathbf{c}})$. This is equivalent to minimizing the \emph{regret} loss function:
\begin{equation} 
    \label{eq:regret_loss}
        \text{regret}(\hat{\mathbf{c}}, \bar{\mathbf{c}}) = f(\mathbf{x}^{\star}(\hat{\mathbf{c}}),\bar{\mathbf{c}}) - f(\mathbf{x}^{\star}(\bar{\mathbf{c}}),\bar{\mathbf{c}}),
\end{equation}
which measures the suboptimality, under ground-truth objective data, of decisions $\mathbf{x}^{\star}(\hat{\mathbf{c}})$ resulting from prediction $\hat{\mathbf{c}}$. 
Since the task amounts to predicting $\hat{\mathbf{c}}$ under ground-truth $\bar{\mathbf{c}}$, a \emph{two-stage} approach is also available which does not require backpropagation through $\mathbf{x}^{\star}$. In the two-stage approach, the loss function $\text{MSE}(\hat{\mathbf{c}},\bar{\mathbf{c}})$ is used to directly target ground-truth parameters, but the final test criteria is still measured by regret. Since the integrated training minimizes regret directly, it generally outperforms the two-stage.

\subsection{Decision-focused learning with nonconvex bilinear programming.}
\label{subsec:bilinear}
The first experiment showcases the ability of folded optimization to be applied in decision-focused learning with \emph{nonconvex} optimization. 
In this experiment, we predict the coefficients of a \emph{bilinear} program
 \begin{subequations}
\label{eq:bilinear}
\begin{align}
    \mathbf{x}^{\star}(\mathbf{c}, \mathbf{d}) = \argmax_{\mathbf{0} \leq \mathbf{x}, \mathbf{y}  \leq \mathbf{1} } &\;\;
    \mathbf{c}^T \mathbf{x} + \mathbf{x}^T \mathbf{Q} \mathbf{y} +\mathbf{d}^T \mathbf{y}    \\
    \texttt{s.t.} \;\; 
    & \sum \mathbf{x} = p, \; \sum \mathbf{y} = q,
\end{align}
\end{subequations}
in which two separable linear programs are confounded by a nonconvex quadratic objective term $\mathbf{Q}$. Costs $\mathbf{c}$ and $\mathbf{d}$ are unknown, while $p$ and $q$ are constants.
Such programs have numerous industrial applications such as optimal mixing and pooling in gas refining \citep{audet2004pooling}. 
Here we focus on the difficulty posed by the problem's form in decision-focused learning, and propose a task in which the unknown parameters $\mathbf{c}$ and $\mathbf{d}$ are correlated with known feature variables and predicted by a 5-layer network. The goal is to predict $\hat{\mathbf{c}}$ and $\hat{\mathbf{d}}$ from features, such that the suboptimality of $\mathbf{x}^{\star}(\hat{\mathbf{c}}, \hat{\mathbf{d}})$ with respect to ground-truth $\mathbf{c}$ and $\mathbf{d}$ is minimized.

It is known that PGD converges to local optima in nonconvex problems \citep{attouch2013convergence}, therefore the \textit{f-PGDb} module specified above is chosen used to backpropagate the solution of \ref{eq:bilinear} in end-to-end training.  Since PGD is not an efficient method for solving the foward-pass mapping \eqref{eq:bilinear}, the \texttt{fold-opt} implementation of this layer uses the Gurobi nonconvex QP solver to find its global optimum. We benchmark against the \emph{two-stage} approach, in which the costs $\mathbf{c}$, and $\mathbf{d}$ are targeted to ground-truth costs by MSE loss and the optimization problem is solved as a separate component from the learning task. In contrast, the integrated  \textit{f-PGDb} layer allows the model to minimize solution regret (i.e., suboptimality) directly as its loss function.

Feature and cost data are generated by the process described in \ref{appendix:experimental}.
In addition, $15$ distinct non-positive semidefinite $\mathbf{Q}$ are randomly generated so that the results of Figure \ref{fig:results_bilinear}(a) are reported on average over all $15$ nonconvex decision-focused learning tasks. Notice  in Figure \ref{fig:results_bilinear} how \textit{f-PGDb} achieves much lower regret for each of the $15$ nonconvex objectives. 

\begin{figure}
    \centering

     \includegraphics[width=0.6\textwidth]{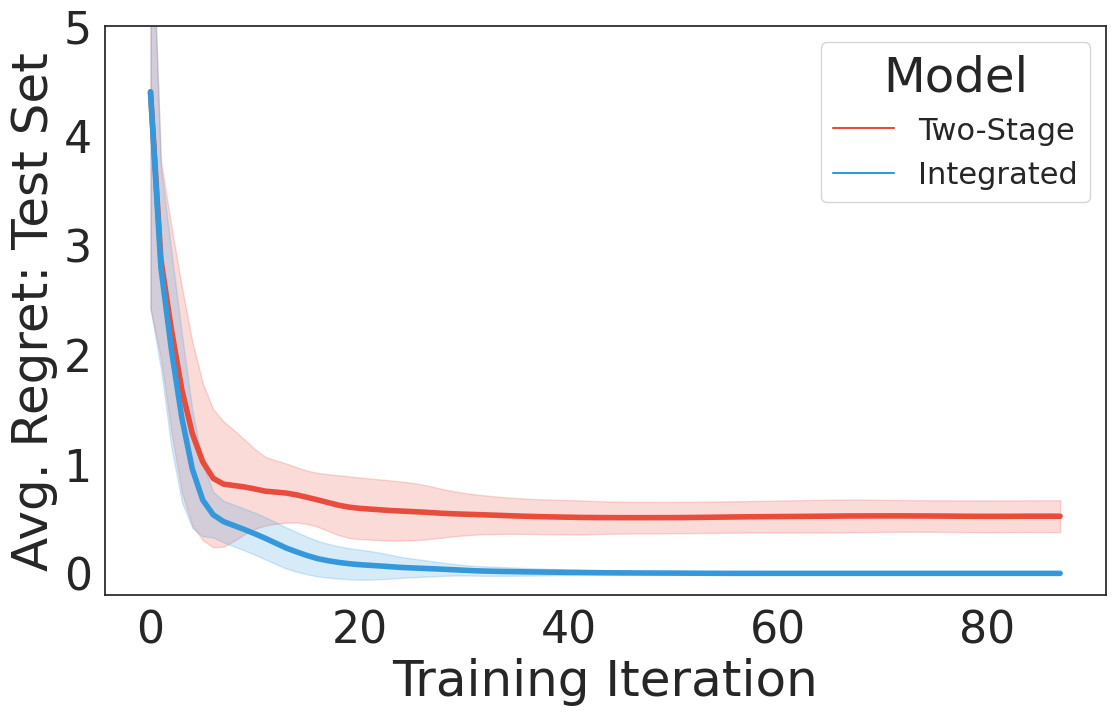} \\
    \caption{Learning Cost Factors in Bilinear Programming.}
    \label{fig:results_bilinear}
\end{figure}

\subsection{Cost Prediction for AC-Optimal Power Flow.}
\label{subsec:optimal_power_flow}

\begin{model}[!t]
	{
    \footnotesize
    \centering
    \begin{mdframed}
	\begin{flalign}
		&{\cO}(\mathbf{c}^l, \mathbf{c}^q) = 
		\textstyle \mathbf{\argmin}_{\mathbf{p^g},\mathbf{v}}\;\;
		 \mathbf{c}^l \cdot \mathbf{p^g} + \mathbf{c}^q \cdot (\mathbf{p^g})^2   && \label{ac_obj} \\
		&\mbox{ \textit{subject to:}} \notag\\
		&\hspace{6pt}
		\dot{v}^{\min}_i \leq v_i \leq \dot{v}^{\max}_i 		
		&& \!\!\!\!\!\forall i \in {\cal N} 		\label{con:2a} \tag{2a}\\
		&\hspace{6pt}
		\text{ -- }\dot{\theta}^\Delta_{ij} \leq \theta_i \text{ -- } \theta_j  \leq \dot{\theta}^\Delta_{ij} 	
		&& \!\!\!\!\!\forall (ij) \in {\cal E}  	 \label{con:2b}\!\!\!\!\! \tag{$\bar{2b}$}\\
		&\hspace{6pt}
		\dot{p}^{g\min}_i \leq p^g_i \leq \dot{p}^{g\max}_i 	
		&& \!\!\!\!\!\forall i \in {\cal N} 		\label{con:3a} \tag{$\bar{3a}$}\\
		&\hspace{6pt}
		\dot{q}^{g\min}_i \leq q^g_i \leq \dot{q}^{g\max}_i 	
		&& \!\!\!\!\!\forall i \in {\cal N} 		\label{con:3b} \tag{3b}\\
		&\hspace{6pt}
		(p_{ij}^f)^2 + (q_{ij}^f)^2 \leq \dot{S}^{f\max}_{ij}			
		&& \!\!\!\!\!\forall (ij) \in {\cal E}	\label{con:4}  \tag{$\bar{4}$}\\
		&\hspace{6pt}
		p_{ij}^f \!=\! \dot{g}_{ij} v_i^2 \text{--}  
		v_i v_j (\dot{b}_{ij} \!\sin(\theta_i \text{--} \theta_j)
		+ \dot{g}_{ij} \!\cos(\theta_i \text{--} \theta_j)\!)	
		&& \!\!\!\!\!\forall (ij)\!\in\! {\cal E} 	\label{con:5a} \tag{$\bar{5a}$}\\
		&\hspace{6pt} 
		q_{ij}^f \!=\! \text{--} \dot{b}_{ij} v_i^2 \text{--}  v_i v_j (\dot{g}_{ij} \!\sin(\theta_i \text{--} \theta_j)
		\text{--} \dot{b}_{ij} \!\cos(\theta_i \text{--} \theta_j)\!)	
		&& \!\!\!\!\!\forall (ij)\!\in\! {\cal E}		\label{con:5b} \tag{5b}\\
		&\hspace{6pt}
		p^g_i \text{ -- } \dot{p}^d_i = \textstyle \sum_{(ij)\in {\cal E}} p_{ij}^f	
		&& \!\!\!\!\!\forall i\in {\cal N} 		\label{con:6a} \tag{$\bar{6a}$}\\
		&\hspace{6pt}
		q^g_i \text{ -- } \dot{q}^d_i = \textstyle 	\sum_{(ij)\in {\cal E}} q_{ij}^f	
		&& \!\!\!\!\!\forall i\in {\cal N} 		\label{con:6b} \tag{6b}
    	\!\!\!\!\!\!\!\!\!\!\!\!\!\!\!\!\!\!\!\!
    	\notag 
	\end{flalign}
	\end{mdframed}
	}
	\captionof{model}{AC Optimal Power Flow (AC-OPF)}
	\label{model:ac_opf}
\end{model}

\begin{figure}
    \centering
    \includegraphics[width=0.6\linewidth]{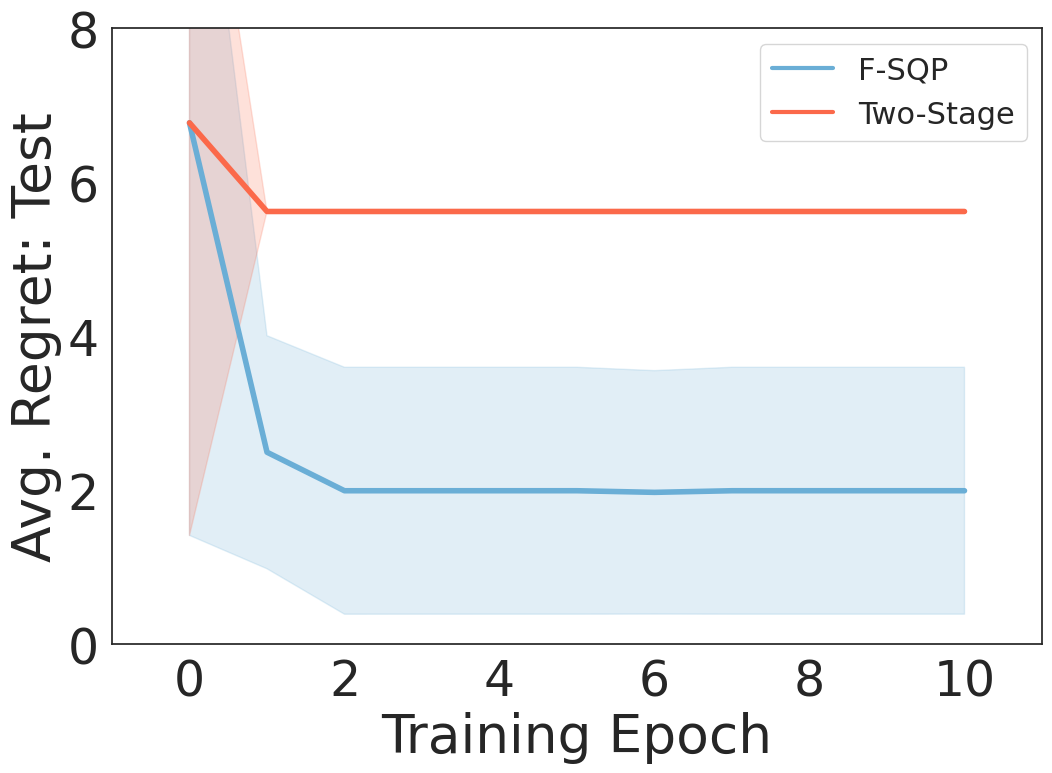} 
   \\
    \caption{Learning cost coefficients of the AC-OPF problem.}
    \label{fig:results_acopf}
\end{figure}

The \emph{AC-Optimal Power Flow} (AC-OPF) problem minimizes the cost of generator dispatch that satisfies the power system's physical and engineering constraints, as shown in Figure \ref{model:ac_opf}. In this learning task, the linear and quadratic power generation costs $\mathbf{c}^q$ and $\mathbf{c}^q$ are unknown and must be inferred from known features, such that the overall price of power generation under ground-truth costs is minimized. All other parameters of the optimization problem are held constant and obtained from the NESTA energy system test case ACOPF-$57$ \cite{coffrin2014nesta}.



A five-layer network is used to predict both sets of cost coefficients in Equation \ref{ac_obj}. Input features are composed of the previous day's temperature, current temperature, and a constant baseline cost vector. Given the predicted cost coefficients, non-convex non-linear solver Interior Point Optimizer, ~\texttt{ipopt}, is used to compute the optimal solution to \ref{ac_obj}. Given that PGD cannot handle non-linear constraints, the \textit{f-SQP} module is used solve for the backpropagated gradients for training by stochastic gradient descent. Figure \ref{fig:results_acopf} shows the relative regret on the test set after each training epoch, compared to a basic two-stage model.

\subsection{Portfolio Prediction and Optimization.}
\label{subsec:portfolio_experiment}
A classic problem which combines prediction with optimization is the Markowitz portfolio problem \cite{zhang2018portfolio}. Here, an investment portfolio must be partitioned to optimize total future return subject to risk constraints, while future asset prices are unknown and must be predicted.  This experiment represents a situation in which \texttt{cvxpy} makes non negligible errors in the forward pass of a problem with nonlinear constraints:
\begin{align}
\label{eq:porfolio}
    \mathbf{x}^\star(\mathbf{c}) = \argmax_{\mathbf{0} \leq \mathbf{x} } 
    \mathbf{c}^T \mathbf{x}\;\;
    \textsl{s.t.} \;
     \mathbf{x}^T \mathbf{V} \mathbf{x} \leq \gamma, \; \sum \mathbf{x} = 1.
\end{align}
This model describes a risk-constrained portfolio optimization where $\bm{V}$ is a covariance matrix, and the predicted cost coefficients $\mathbf{c}$ represent assets prices \citep{elmachtoub2020smart}. 
A  $5$-layer ReLU network is used to predict future prices $\bm{c}$ from exogenous feature data, and trained to minimize regret (the difference in profit between optimal portfolios under predicted and ground-truth prices) by integrating Problem (\ref{eq:porfolio}).
The folded \textit{f-SQP} layer used for this problem employs Gurobi QCQP solver in its forward pass. This again highlights the ability of \texttt{fold-opt} to accommodate a highly optimized blackbox solver. 

Figure \ref{fig:results_portfolio}   shows test set regret  throughout training, on three synthetically generated datasets of different nonlinearity degrees, following exactly the experimental settings of \cite{elmachtoub2020smart}. 
Notice the accuracy improvements of  \texttt{fold-opt} over ~\texttt{cvxpy}. 
Such dramatic differences can be explained by non-negligible errors made in \texttt{cvxpy}'s forward pass optimization on some problem instances, which occurs regardless of error tolerance settings; this may be due to ill-conditioning of the quadratic constraint in \eqref{eq:porfolio}. In contrast, Gurobi agrees to machine precision with a custom SQP solver, and solves about $50 \%$ faster than \texttt{cvxpy}. This shows the importance of highly accurate optimization solvers for accurate end-to-end training.

\begin{figure}
    \label{fig:portfolio_results}
    \centering
     \includegraphics[width=0.99\linewidth]{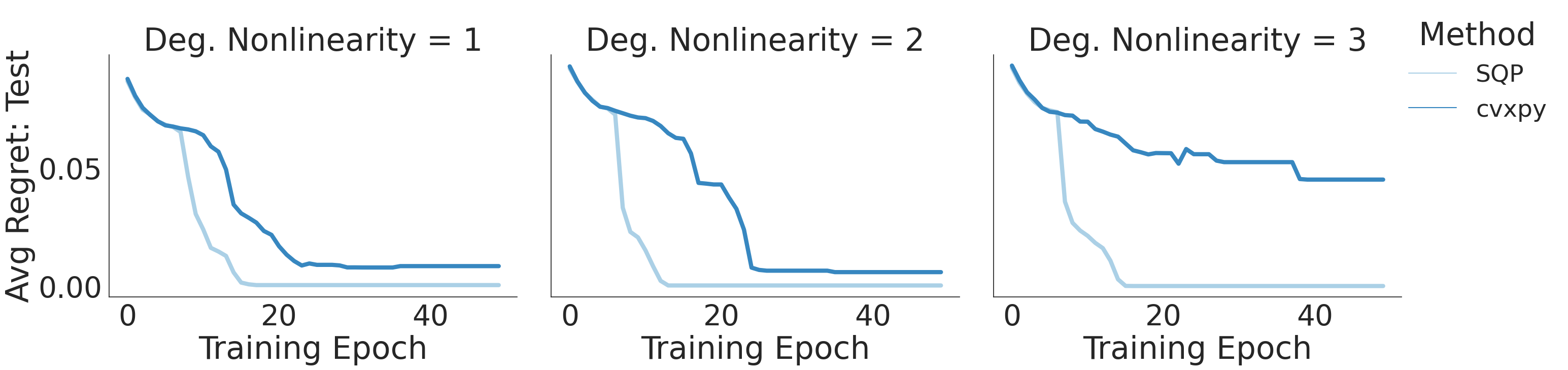} \\
    \caption{Learning asset prices for portfolio optimization.}
    \label{fig:results_portfolio}
\end{figure}

\subsection{Enhanced Total Variation Denoising.}
\label{subsec:denoising_experiment}
A classic application of proximal optimization models a denoising problem

\begin{equation}
    \label{eq:denoising}
    \mathbf{x}^{\star}(\mathbf{d}, \mathbf{D}) = \argmin_{\mathbf{x}} \;\;
    \frac{1}{2} \| \mathbf{x}-\mathbf{d} \|^2 +  \lambda \| \mathbf{D} \mathbf{x} \|_1,
\end{equation}


which seeks to recover the true signal $\mathbf{x}^{\star}$ from  a noisy input $\mathbf{d}$ and is often best handled by variants of Dual Proximal Gradient Descent \cite{beck2017first}. Typically, $\mathbf{D}$ is a pairwise differencing matrix so that $\| \mathbf{D} \mathbf{x} \|_1$ represents total variation. The objective function, which balances a combination of distance to the input signal $\mathbf{d}$ with a penalty on variation, aims to find $\mathbf{x}^{\star}$ which removes extraneous noise from $\mathbf{d}$. 
Here we initialize $\mathbf{D}$ to the classic differencing matrix and \emph{learn} a better operator by treating $\mathbf{D}$ as a learnable parameter.

Training data follows the experimental settings of \cite{amos2019optnet}, in which a set of $1$D signals is treated as target data and then perturbed by Gaussian noise to generate their corresponding noisy input data $\mathbf{d}$. MSE loss is used to target the true signals while $\mathbf{D}$ is learned. Figure \ref{fig:denoiser_plots}(a) shows MSE on the test set throughout training due to \textit{f-FDPG} for various choices of $\lambda$. Figure \ref{fig:denoiser_plots}(b)
shows comparable results from the differentiable QP framework of \cite{amos2019optnet}, which converts the problem \eqref{eq:denoising} to an equivalent QP problem:

\begin{subequations}
\label{eq:denoiser_QP}
\begin{align}
    \mathbf{x}^{\star}(\mathbf{D}) = \argmin_{\mathbf{x},\mathbf{t}} &\;\;
    \frac{1}{2} \| \mathbf{x}-\mathbf{d} \|^2 + \lambda \overrightarrow{\mathbf{1}} \mathbf{t}    \\
    \textit{s.t.} \;\;
    &    \;\;\;\;\;\;\;\;\;\;\;   \mathbf{D} \mathbf{x} \leq \mathbf{t} \\
    & -\mathbf{t} \leq \mathbf{D} \mathbf{x} 
\end{align}
\end{subequations}

in order to differentiably solve the denoising problem in \texttt{qpth}. Small differences in these results likely stem from solver error tolerance in the forward pass of each method. However, \textit{f-FDPG} computes $\mathbf{x}^{\star}(\mathbf{D})$ up to $40$ times faster, by using an optimization method \eqref{eq:FDPG_fwd} which is well-chosen for efficiently solving the denoising problem in its original form \eqref{eq:denoising}. 

\begin{figure}
     \centering
     \begin{subfigure}[b]{0.51\linewidth}
         \centering
         \includegraphics[width=\textwidth]{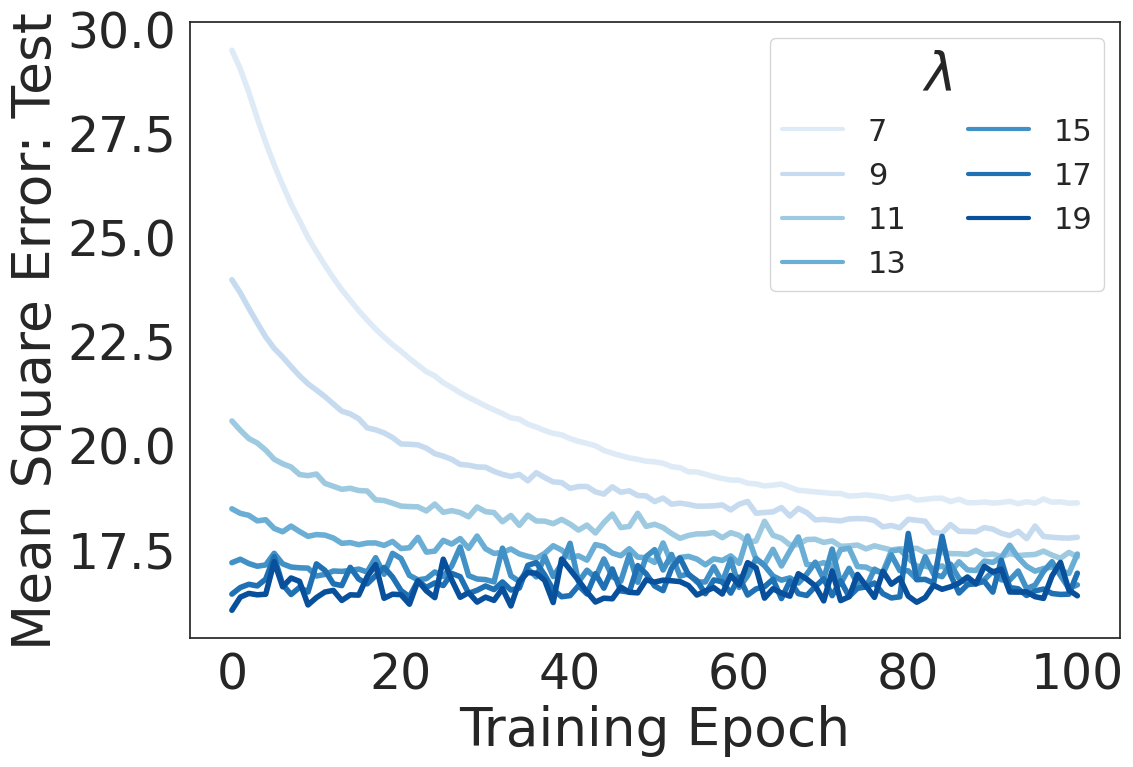}
         \caption{\textit{f-FDPG}}
     \end{subfigure}
     \hfill
     \begin{subfigure}[b]{0.48\linewidth}
         \centering
         \includegraphics[width=\textwidth]{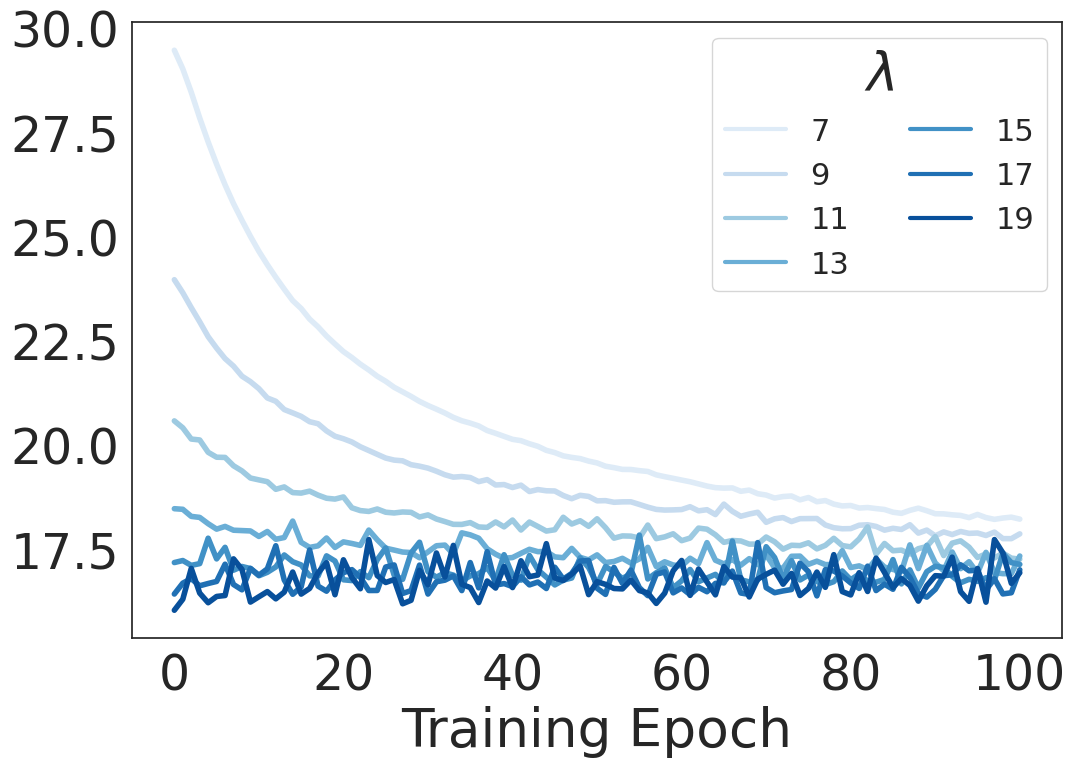}
         \caption{\texttt{qpth}}
     \end{subfigure}
     \hfill

    \caption{Enhanced Denoising Task: Test Set Loss}
     \label{fig:denoiser_plots}
\end{figure}

\subsection{Mutilabel Classification on CIFAR100.}
\label{subsec:multilabel_experiment}
Since the differential fixed-point conditions \eqref{eq:Lemma_fixedpt} depend on the chosen optimization method \eqref{eq:opt_iteration}, we compare the effect of different backpropagation rules in \texttt{fold-opt}, based on alternative choices of \eqref{eq:opt_iteration}.  This experiment compares the backpropagation of both \textit{f-PGDa} and \textit{f-SQP}  with that of \texttt{cvxpy}, since both PGD and SQP methods are suitable for solving the optimization  \eqref{eq:topk-lp}. Importantly, each \texttt{fold-opt} layer uses the same forward pass, implemented in \texttt{cvxpy}. This allows any potential descrepancies in empirical results to be attributed to differences in the backpropagation model.

The experimental task, adapted from \citep{Berrada2018SmoothLF}, learns a smooth top-5 classification model on noisy CIFAR-100. The optimization below maps image feature embeddings $\mathbf{c}$ from DenseNet  $40$-$40$ \citep{huang2017densely}, to smoothed top-$k$ binary class indicators (see \ref{appendix:experimental} 
for more details): 
\begin{align}
\label{eq:topk-lp}
    \mathbf{x}^{\star}(\mathbf{c}) \!=\! 
    \argmax_{\mathbf{0} \leq \mathbf{x} \leq \mathbf{1}} &\;\;
    \mathbf{c}^T \mathbf{x} + \sum_i x_i \log x_i \;\;
    \textsl{s.t.} \;
    \sum \mathbf{x} = k 
\end{align}
Figure \ref{fig:ml_plot} 
shows that all three models have indistinguishable classification accuracy throughout training, even after 30 epochs of training on 45k samples. On the other hand, the more sensitive test set shows marginal accuracy divergence between all three methods after a few epochs. This corresponds with a slightly less consistent increase in accuracy throughout training, in which none of the methods holds a clear advantage.

\begin{figure}
    \centering
    \includegraphics[width=0.8\linewidth]{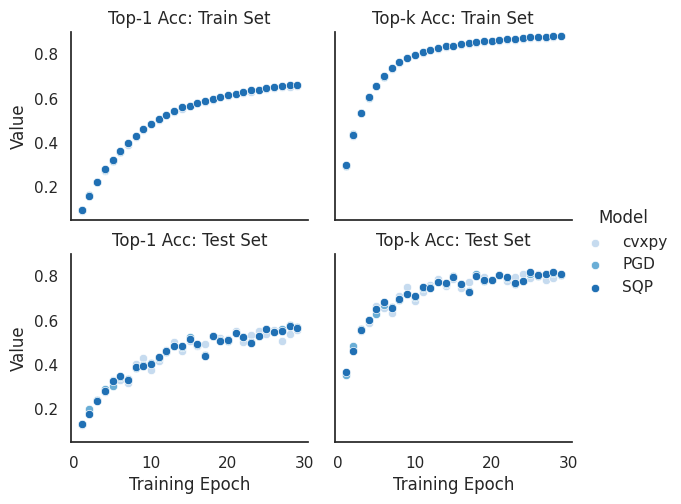} 
   \\
    \caption{Test and train set accuracy while training multilabel classification on CIFAR-100.}
    \label{fig:ml_plot}
\end{figure}

\section{Conclusions}

This paper introduced folded optimization, a framework for generating efficient and analytically differentiable optimization solvers from unrolled implementations. Theoretically, folded optimization was justified by a novel analysis of unrolled optimization at a precomputed optimal solution, which showed that its backward pass is equivalent to solution of a solver's differential fixed-point conditions, specifically by fixed-point iteration. This allowed for the convergence analysis of the backward pass, and evidence that the convergence could be improved by using superior linear system solvers. The paper showed that folded optimization offers substantial advantages over existing both unrolled optimization and existing differentiable optimization frameworks, including modularization of the forward and backward passes and the ability to handle nonconvex optimization.

\section*{Acknowledgements}
This research is partially supported by NSF grants 2345528, 2334936, 2334448 and NSF CAREER Award 2143706. Fioretto is also supported by an Amazon Research Award and a Google Research Scholar Award. Its views and conclusions are those of the authors only.

\bibliographystyle{elsarticle-num} 
\bibliography{bib}

\appendix
\include{appendix}

\end{document}

%% file: appendix.tex

\section{Optimization Models}
\label{appendix:models}

\paragraph{Soft Thresholding Operator}

The soft thresholding operator defined below arises in the solution of denoising problems proximal gradient descent variants as the proximal operator to the $\| \cdot \|_1$ norm:
\[
    \mathcal{T}_{\lambda}(\mathbf{x}) = \left[  | \mathbf{x} | - \lambda \mathbf{e}  \right]_{+}  \cdot \textit{sgn} (\mathbf{x})
\]

\paragraph{Fast Dual Proximal Gradient Descent}
The following is an FDPG implementation from \cite{beck2017first}, specialized to solve the denoising problem 
\[
    \mathbf{x}^{\star}(\mathbf{D}) = \argmin_{\mathbf{x}} \;\;
    \frac{1}{2} \| \mathbf{x}-\mathbf{d} \|^2 +  \lambda \| \mathbf{D} \mathbf{x} \|_1,
\]
of Section \ref{sec:Experiments}. Letting $\mathbf{u}_k$ be the primal solution iterates, with $t_0=1$ and arbitrary $\mathbf{w}_0 = \mathbf{y}_0$:
\begin{subequations}
\label{eq:FDPG_fwd}
\begin{align}
    \mathbf{u}_k &= \mathbf{D}^T \mathbf{w}_k + \mathbf{d} \\
    \mathbf{y}_{k+1} &= \mathbf{w}_k - \frac{1}{4} \mathbf{D} \mathbf{u}_k + \frac{1}{4} \mathcal{T}_{4 \lambda} ( \mathbf{D} \mathbf{u}_k - 4 \mathbf{w}_k ) \\
    t_{k+1} &= \frac{1 + \sqrt{1+4 t_k^2}}{2} \\
    \mathbf{w}_{k+1} &= \mathbf{y}_{k+1} + \left( \frac{t_k - 1}{t_{k+1}} \right) (\mathbf{y}_{k+1} - \mathbf{y}_k)
\end{align}
\end{subequations}

\paragraph{Quadratic Programming by ADMM}
A Quadratic Program is an optimization problem with convex quadratic objective and linear constraints. The following ADMM scheme of \cite{boyd2011distributed} solves any quadratic programming problem of the standard form:
\begin{subequations}
\begin{align}
    \argmax_{x} &\;\;
    \frac{1}{2} \mathbf{x}^T \mathbf{Q} \mathbf{x} + \mathbf{p}^T \mathbf{x}\\
    \textit{s.t.} \;\; & \mathbf{A}\mathbf{x} = \mathbf{b}\\
     &\mathbf{x} \geq \mathbf{0}
\end{align}
\end{subequations}
by declaring the operator splitting
\begin{subequations}
\label{app:QP_ADMM}
\begin{align}
    \argmax_{\mathbf{x}} &\;\;
    f(\mathbf{x}) + g(\mathbf{z})\\
    \textit{s.t.} \;\; & \mathbf{x}=\mathbf{z}
\end{align}
\end{subequations}
with $  f(\mathbf{x}) = \frac{1}{2} \mathbf{x}^T \mathbf{Q} \mathbf{x} + \mathbf{p}^T \mathbf{x}  $, $dom(f) = \{ \mathbf{x}: \mathbf{A}\mathbf{x} = \mathbf{b} \} $, $  g(\mathbf{x}) = \delta( \mathbf{x} \geq 0 )  $ and where $\delta$ is the indicator function. 

This results in the following ADMM iterates: 
\begin{enumerate}
    \item Solve $\begin{bmatrix} \mathbf{P}+\rho \mathbf{I}&\mathbf{A}^T\\\mathbf{A}&\mathbf{0}\end{bmatrix} \begin{bmatrix} \mathbf{x}_{k+1} \\ \pmb{\nu}\end{bmatrix}   =  \begin{bmatrix}-\mathbf{q} + \rho(\mathbf{z}_k - \mathbf{u}_k)\\\mathbf{b}\end{bmatrix}  $
    \item $\mathbf{z}_{k+1} = (\mathbf{x}_{k+1} + \mathbf{u}_k)_+$
    \item $\mathbf{u}_{k+1} = \mathbf{u}_k + \mathbf{x}_{k+1} - \mathbf{z}_{k+1}$
\end{enumerate}
Where $(1)$ represents the KKT conditions for equality-constrained minimization of $f$,  $(2)$ is projection onto the positive orthant, and  $(3)$ is the dual variable update.

\paragraph{Sequential Quadratic Programming}
For an optimization mapping defined by Problem (\ref{eq:opt_generic}) where $f$, $g$ and $h$ are continuously differentiable, define the operator $\mathcal{T}$ as:
\begin{subequations}
\begin{align}
    \mathcal{T}(\mathbf{x},\pmb{\lambda}) = \argmin_{\mathbf{d}} &\;\;
     \nabla f(\mathbf{x})^T \mathbf{d} + \mathbf{d}^T  \nabla^2 \mathcal{L}(\mathbf{x},\pmb{\lambda}) \mathbf{d} \\
    \texttt{s.t.} \;\; &
    h(\mathbf{x}) + \nabla h(\mathbf{x})^T \mathbf{d} = \mathbf{0}\\
     &g(\mathbf{x}) + \nabla g(\mathbf{x})^T \mathbf{d} \leq \mathbf{0}
\end{align}
\end{subequations}
where dependence of each function on parameters $\mathbf{c}$ is hidden. The function $\mathcal{L}$ is a Lagrangian function of Problem (\ref{eq:opt_generic}).  Then given initial estimates of the primal and dual solution $(x_0,\lambda_0)$, sequential quadratic programming is defined by 
\begin{subequations}
\begin{align}
    (\mathbf{d},\pmb{\mu}) =  \mathcal{T}(\mathbf{x}_{k},\pmb{\lambda}_{k}) \label{line:1}\\
    \mathbf{x}_{k+1} = \mathbf{x}_k + \alpha_k \mathbf{d} \\
    \pmb{\lambda}_{k+1} =   \alpha_k( \pmb{\mu} - \pmb{\lambda}_k) 
\end{align}
\end{subequations}
Here, the inner optimization $\mathcal{O} = \mathcal{T}$ as in Section \ref{sec:unfolding}.

\section{Experimental Details}
\label{appendix:experimental}
Additional details for each experiment of Section \ref{sec:Experiments} are described in their respective subsections below. Note that in all cases, the machine learning models compared in Section \ref{sec:Experiments} use identical settings within each study, with the exception of the optimization components being compared.

\subsection{Nonconvex Bilinear Programming}

\paragraph{Data generation} Data is generated as follows for the nonconvex bilinear programming experiments. Input data consists of  $1000$ points $\in \mathbb{R}^{10}$ sampled uniformly in the interval $\left[-2, 2\right]$. To produce targets, inputs are fed into a randomly initialized $2$-layer neural network  with $\tanh$ activation, and gone through a nonlinear function $x \cos{2x}+ \frac{5}{2} \log{\frac{x}{x+2}} + x^2\sin{4x}$ to increase the nonlinearity of the mapping between inputs and targets. Train and test sets are split  $90 / 10$. 

\paragraph{Settings} A 5-layer NN with ReLU activation trained to predict cost $\mathbf{c}$ and $\mathbf{d}$. We train model with Adam optimizer on learning rate of $10^{-2}$ and batch size 32 for 5 epochs. Nonconvex objective coefficients Q are pre-generated randomly with 15 different seeds. Constraint parameters are chosen arbitrarily as $p=1$ and $q=2$. The average solving time in Gurobi is $0.8333$s, and depends per instance on the predicted parameters $\mathbf{c}$ and $\mathbf{d}$. However the average time tends to be dominated by a minority of samples which take up to $\sim 3$ min. This issue is mitigated by imposing a time limit in solving each instance. While the correct gradient is not guaranteed under early stopping, the overwhelming majority of samples are fully optimized under the time limit, mitigating any adverse effect on training. Differences in training curves under $10$s and $120$s timeouts are negligible due to this effect; the results reported use the $120$s timeout.

\subsection{AC-Optimal Power Flow}

\paragraph{Data Generation} A 57-node power system is used to generate our dataset. Specifications of generators, demand loads, and buses are adapted directly from the NESTA energy system test case \cite{coffrin2014nesta}. Cost coefficients are randomly perturbed from the original generator costs and altered by a non-linear function of the temperature variations $x \times (1 + \frac{| t_{previous} - t_{current} |}{100})$. Temperature variations are represented by the previous day temperature $t_{previous}$, sampled uniformly in the interval [20, 110], and the current day temperature $t_{current}$ which is computed by sampling the change in temperature normally with a mean of 0 and variation of 20. The demand loads are also modified by a non-linear function of the temperature variations $x + |\frac{t_{current} - 65}{45}|$. To train the model, $1000$ points were sampled to create the dataset with a $90$/$10$ train/test split. 

\paragraph{Settings} A five-layer ReLU network with hidden layer size $64$ is trained to predict generator costs $\mathbf{c} \in \mathbb{R}^{7 \times 3}$  using SGD optimizer with learning rate $10^{-2}$ and batch size $32$.

\subsection{Portfolio Optimization}

\paragraph{Data Generation} The data generation follows exactly the prescription of Appendix D in \cite{elmachtoub2020smart}. Uniform random feature data are mapped through a random nonlinear function to create synthetic price data for training and evaluation. A random matrix is used as a linear mapping, to which nonlinearity is introduced by exponentiation of its elements to a chosen degree. The studies in Section \ref{sec:Experiments} use degrees $1$, $2$ and $3$.

\paragraph{Settings} A five-layer ReLU network is trained to predict asset prices $\mathbf{c} \in \mathbb{R}^{20}$ using Adam optimizer with learning rate $10^{-2}$ and batch size $32$.

\subsection{Enhanced Denoising}

\paragraph{Data generation} The data generation follows \cite{amos2019optnet}, in which $10000$ random $1D$ signals of length $100$ are generated and treated as targets. Noisy input data is generated by adding random perturbations to each element of each signal, drawn from independent standard-normal distributions. A $90 / 10$ train/test split is applied to the data.

\paragraph{Settings} A learning rate of $10^{-3}$ and batch size $32$ are used in each training run. Each denoising model is initialized to the classical total variation denoiser by setting the   learned matrix of parameters $\mathbf{D} \in \mathbb{R}^{99 \times 100}$ to the differencing operator, for which $D_{i,i} = 1$ and $D_{i,i+1} = -1 \;\; \forall i$ with all other values $0$.

\subsection{Multilabel Classification}

\paragraph{Dataset} We follow the  experimental settings and implementation  provided by \cite{Berrada2018SmoothLF}. Each model is evaluated on the noisy top-5 CIFAR100 task. CIFAR-100 labels are organized into 20 “coarse” classes,
each consisting of 5 “fine” labels. With some probability, random noise is added to each label by resampling from the set of “fine”
labels. The $50$k data samples are given a $90 / 10$ training/testing split. 

\paragraph{Settings}
The DenseNet $40$-$40$ architecture is trained by SGD optimizer with learning rate $10^{-1}$ and batch size $64$ for $30$ epochs to minimize a cross-entropy loss function.